\documentclass[lettersize,journal]{IEEEtran}
\usepackage{amsmath,amsfonts}

\DeclareMathOperator*{\argmin}{arg\,min}

\usepackage{algorithm}
\usepackage{algpseudocode}
\usepackage{array}
\usepackage[caption=false,font=normalsize,labelfont=sf,textfont=sf]{subfig}
\usepackage{textcomp}
\usepackage{stfloats}
\usepackage{url}
\usepackage{verbatim}
\usepackage{graphicx}
\usepackage{subcaption}
\usepackage{cite}
\usepackage{amsthm} 
\usepackage{booktabs}

\theoremstyle{plain}
\newtheorem{theorem}{Theorem}[section]

\theoremstyle{definition}

\theoremstyle{remark}
\newtheorem{remark}{Remark}

\usepackage{hyperref}
\hypersetup{
    colorlinks=true,
    linkcolor=blue,
    filecolor=blue,      
    urlcolor=blue,
    citecolor=blue,
}

\begin{document}

\title{Coverage-Recon: Coordinated Multi-Drone Image Sampling with Online Map Feedback} 


\author{Muhammad Hanif$^{1\href{https://orcid.org/0009-0001-0475-8711}{\includegraphics[scale=0.2]{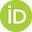}}}$,~\IEEEmembership{Graduate Student         Member,~IEEE}, 
        Reiji Terunuma$^{1\href{https://orcid.org/0009-0009-1252-0592}{\includegraphics[scale=0.2]{images/orcid_32x32.png}}}$,~\IEEEmembership{Graduate Student Member,~IEEE},
        Takumi Sumino$^{1\href{https://orcid.org/0009-0006-2382-9841}{\includegraphics[scale=0.2]{images/orcid_32x32.png}}}$,  
        Kelvin Cheng$^{2\href{https://orcid.org/0000-0003-2779-9150}{\includegraphics[scale=0.2]{images/orcid_32x32.png}}}$, 
        and Takeshi Hatanaka$^{1\href{https://orcid.org/0000-0003-3143-121X}{\includegraphics[scale=0.2]{images/orcid_32x32.png}}}$,~\IEEEmembership{Senior Member,~IEEE}%
        
\thanks{This work was supported in part by the Japan Society for the Promotion of Science (JSPS) KAKENHI under Grant 24K00906.}%
\thanks{$^{1}$Muhammad Hanif, Reiji Terunuma, Takumi Sumino, and Takeshi Hatanaka are with the School of Engineering, Institute of Science Tokyo, Tokyo 152-8552, JAPAN. 
        {\tt\small \{hanif@hfg., terunuma@hfg., sumino@hfg., hatanaka@\}sc.e.titech.ac.jp}}%
\thanks{$^{2}$Kelvin Cheng is with the Rakuten Institute of Technology, Rakuten Group, Inc., Tokyo, JAPAN.
        {\tt\small kelvin.cheng@rakuten.com}}%
}


\markboth{IEEE TRANSACTIONS ON CONTROL SYSTEMS TECHNOLOGY}%
{Shell \MakeLowercase{\textit{et al.}}: A Sample Article Using IEEEtran.cls for IEEE Journals}


\maketitle

\begin{abstract}  
This article addresses collaborative 3D map reconstruction using multiple drones. Achieving high-quality reconstruction requires capturing images of keypoints within the target scene from diverse viewing angles, and coverage control offers an effective framework to meet this requirement. Meanwhile, recent advances in real-time 3D reconstruction algorithms make it possible to render an evolving map during flight, enabling immediate feedback to guide drone motion. Building on this, we present Coverage-Recon, a novel coordinated image sampling algorithm that integrates online map feedback to improve reconstruction quality on-the-fly. In Coverage-Recon, the coordinated motion of drones is governed by a Quadratic Programming (QP)-based angle-aware coverage controller, which ensures multi-viewpoint image capture while enforcing safety constraints. The captured images are processed in real time by the NeuralRecon algorithm to generate an evolving 3D mesh. Mesh changes across the scene are interpreted as indicators of reconstruction uncertainty and serve as feedback to update the importance index of the coverage control as the map evolves. The effectiveness of Coverage-Recon is validated through simulation and experiments, demonstrating both qualitatively and quantitatively that incorporating online map feedback yields more complete and accurate 3D reconstructions than conventional methods.  

\noindent\textbf{Project page:} \url{https://htnk-lab.github.io/coverage-recon/}  
\end{abstract}

\begin{IEEEkeywords}
Coverage Control, 3D Reconstruction, Multi-Drone System, NeuralRecon, QP-based Control.
\end{IEEEkeywords}

\section{Introduction}


\IEEEPARstart{M}{ulti-robot} systems (MRS) have gained significant attention as a promising robotic framework, offering efficient, scalable, and resilient solutions compared to single-robot systems~\cite{cortes2017coordinated}. They have been applied in diverse domains, including precision agriculture~\cite{edmonds2021efficient, albani2019field},
environmental monitoring~\cite{toyomoto2025constraint, martin2023predictive}, and safety-critical missions such as wildfire tracking, search-and-rescue, and infrastructure inspection~\cite{seraj2020coordinated, alotaibi2019lsar, queralta2020collaborative, 
daftry2015building}. 
Beyond these traditional domains, collaborative 3D map reconstruction has also emerged as a particularly impactful application, where multiple drones capture images from diverse viewpoints to build high-quality 3D models that are crucial for spatial understanding and decision-making~\cite{comba2018unsupervised, gupta2018application}.

In collaborative 3D map reconstruction, coordination during the image sampling process is critical for achieving accurate and reliable reconstructions. Here, there are two important aspects to be considered \cite{lu2024angle}. Firstly, each region of the scene should be observed from multiple viewpoints to ensure robustness and geometric accuracy~\cite{maboudi2023review}, which relies on overlapping camera poses ~\cite{smith2018aerial}. Secondly, the one-time-visit constraint must be satisfied~\cite{shimizu2021angle}, meaning that images are systematically acquired so the entire environment is effectively sampled within a single mission.   

\begin{figure*}[t]
    \centering
    \includegraphics[width=0.78\textwidth]{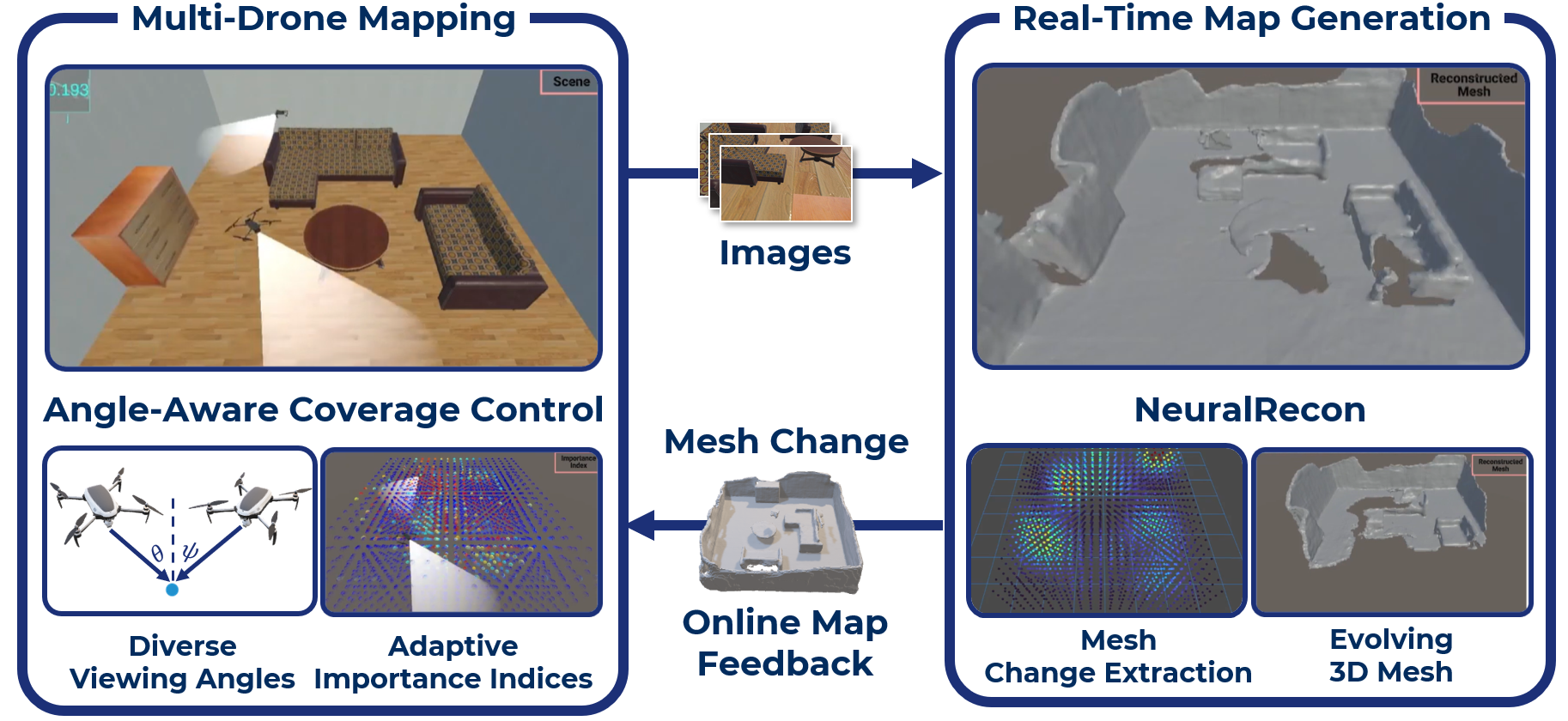}
    \caption{Coverage-Recon enables autonomous multi-drone reconstruction through a closed feedback loop between image sampling and mapping. Left: drones use angle-aware coverage control with a QP-based controller to capture diverse viewpoints. Right: NeuralRecon~\cite{sun2021neuralrecon} generates an evolving 3D mesh that guides drones to focus on under-reconstructed regions.}
    \label{fig:coverage-recon}
\end{figure*}

Some early approaches relied on preset path planning strategies, such as lawnmower or zigzag patterns~\cite{tmuvsic2020current, torres2016coverage}. While simple, these methods do not incorporate camera orientation planning, leading to poor viewpoint diversity and uneven sampling, such as oversampling rooftops while undersampling facades~\cite{maboudi2023review}. They also lack the flexibility to adapt to unexpected disturbances during flight. To address these issues, adaptive methods have been proposed, including viewpoint selection~\cite{delmerico2018comparison,bircher2016receding} to maximize information gain and trajectory optimization~\cite{song2017online, roberts2017submodular} to reduce reconstruction uncertainty. However, most rely on prior or partial maps knowledge and are often implemented in centralized frameworks, which limits scalability in multi-drone deployments.

As a distributed alternative, coverage control has been widely studied for coordinating multiple agents. Classical Voronoi-based methods~\cite{cortes2004coverage, schwager2011eyes} drive robots toward optimal configurations according to prescribed importance indices defined over a space. However, these methods often lead to static configurations and cannot ensure continuous coverage over time. Persistent coverage control addresses this by using time-varying importance indices~\cite{hubel2008coverage,dan2020control}, but it remains limited to 2D and lacks the viewpoint diversity needed for 3D reconstruction. To overcome this, angle-aware coverage integrates drone position and camera orientation into a five-dimensional formulation~\cite{shimizu2021angle, hanif2024efficient}, enabling diverse viewpoints and improved reconstruction~\cite{suenaga2022experimental}. More recently, angle-aware coverage has also been extended to include gimbal-mounted cameras, allowing explicit orientation control independent of body motion~\cite{lu2024angle}. 

While angle-aware coverage strategies broaden the range of achievable viewpoints, their importance indices evolve independently of reconstruction progress and therefore cannot adapt to the evolving state of the map. This stems from traditional offline reconstruction pipelines, such as Structure from Motion (SfM) and Multi-View Stereo (MVS)~\cite{schonberger2016structure,  furukawa2009accurate},
which require the complete set of images before reconstruction begins. As a result, map quality cannot be assessed during flight, leaving the sensing mission blind until post-processing. To mitigate this drawback, some works have introduced multi-stage planning strategies that incorporate partial reconstruction results into subsequent flights. For example, Plan3D~\cite{hepp2018plan3d} uses an initial exploratory mission to generate a sparse model, followed by optimized paths for refinement, while CEO-MLCPP~\cite{lee2022ceo} applies model-based optimization to select additional viewpoints after an initial mission. These methods improve final map quality but require multiple deployments, increasing time, cost, and complexity, making them less suitable for time-critical or resource-limited applications.

Recent breakthroughs in learning-based 3D reconstruction have introduced a new paradigm for online, iterative map generation. Methods such as NeuralRecon~\cite{sun2021neuralrecon}, Instant-NGP~\cite{muller2022instant}, VisFusion~\cite{gao2023visfusion}, and Gaussian Splatting~\cite{kerbl20233d} use neural feature fusion and differentiable rendering to incrementally update 3D models in near real time. These advances make it possible to render a 3D map during flight, providing immediate feedback to guide drone motion.
Building on these advances, recent studies have explored using real-time 3D reconstruction as feedback for motion planning. SplatNav~\cite{chen2024safer} and SAFER-Splat~\cite{chen2024splat} integrate reconstruction updates into navigation pipelines, enabling drones to react to evolving scene geometry and improve safety. NeurAR~\cite{ran2023neurar} leverages neural uncertainty to refine flight paths in under-sampled regions. Other frameworks, including RT-GuIDE~\cite{tao2024rt}, GS-Planner~\cite{jin2024gs}, and HGS-Planner~\cite{xu2024hgs}, incorporate reconstruction feedback to guide view selection and trajectory planning. While these studies highlight the potential of integrating online reconstruction with planning, applying such feedback-driven strategies to multi-robot coverage control remains largely unexplored.


In this article, we present Coverage-Recon, a multi-drone coordinated image sampling algorithm that integrates online map feedback to improve 3D reconstruction quality in real time (see Fig.~\ref{fig:coverage-recon}). The problem is formulated as an angle-aware coverage control task, where multiple drones collaboratively explore the scene while capturing images from diverse viewpoints. These images are processed on-the-fly using NeuralRecon~\cite{sun2021neuralrecon} to incrementally generate an evolving 3D mesh. Variations in the mesh are treated as indicators of reconstruction uncertainty and fed back to dynamically refine the importance index, guiding drones to focus on regions requiring further observations. A Quadratic Programming (QP)-based controller is then designed to guarantee sampling performance by constraining the decay rate of the objective function, while enforcing safety through Control Barrier Functions (CBFs)~\cite{ames2016control}. The Coverage-Recon framework is validated through Unity--ROS2 simulations and real-world experiments, demonstrating that incorporating online feedback yields more accurate and complete reconstructions, with higher F-Score performance compared to baseline methods.

In summary, the main contributions are:
\begin{enumerate}
    \item An extended angle-aware coverage control formulation~\cite{lu2024angle} with altitude control, expanding the control inputs to five dimensions and introducing a modified sensing performance function.
    \item A novel mesh-change quantification method for dynamically updating importance indices, with comparisons to M3C2-based approaches~\cite{lague2013accurate}.
    \item Integration of the NeuralRecon pipeline~\cite{sun2021neuralrecon} with coverage control, establishing a closed feedback loop between image sampling and reconstruction.
    \item Validation through simulation and real-world experiments, demonstrating significant improvements in reconstruction quality both qualitatively and quantitatively.
\end{enumerate}


Finally, an earlier version of this work appeared in conference proceedings~\cite{hanif2025impact}. 
Differently from the conference version, the current article presents multi-drone simulations, quantitative evaluation through a comparative study against benchmark baselines~\cite{torres2016coverage, dan2020control, shimizu2021angle, lu2024angle}, and experimental validation, along with a refined framework including methodological improvements.

\section{Problem Settings}
\label{sec:problem_setting}

This chapter formulates the task of coordinated multi-drone image sampling for 3D reconstruction as an \emph{angle-aware coverage control problem}.  
We begin by introducing the drones and the assumptions on the target region to be reconstructed, which is modeled as a virtual field along with its corresponding geometric relationships.  
We then describe how drone sensing performance is quantified when observing the scene, and finally, we present the importance index and global objective function that will be used in the subsequent control design.

\subsection{Drones, Virtual Field, and Geometry}
\label{subsec:2a}


\begin{figure}[t]
    \centering
    \includegraphics[width=0.9\columnwidth]{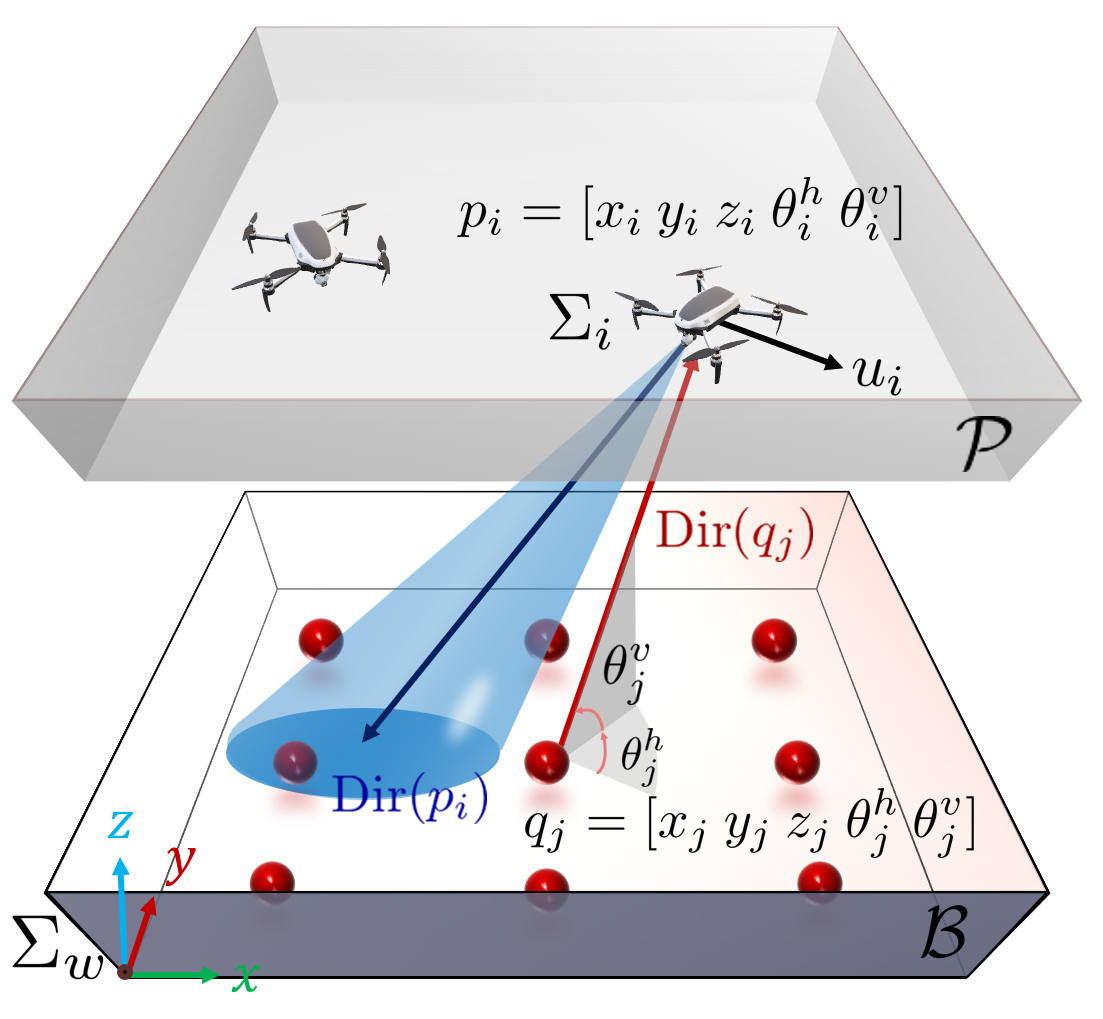}
    \caption{Illustration of the geometric relationship between drones and observation points. Drones in flight region $\mathcal{P}$ observe target region $\mathcal{B}$, adjusting camera yaw $\theta^h_i$ and pitch $\theta^v_i$ to capture diverse viewpoints for 3D reconstruction.}
    \label{fig:prob_illust}
\end{figure}

Consider a group of $n$ drones, indexed by $\mathcal{I} := \{1,2,\ldots,n\}$, that fly in a three-dimensional Euclidean space. Let $\Sigma_w$ denote the world frame.
Each drone is modeled as a rigid body with its own body frame $\Sigma_i$, attached to the $i$-th drone’s camera.
The body frame is defined such that its $z$-axis aligns with the optical axis of the onboard camera.
Each drone is also equipped with a gimbal that enables independent control of the camera’s yaw and pitch angles.

The state of drone $i$ is represented by
\[
p_i = [x_i, y_i, z_i, \theta^h_i, \theta^v_i]^\top \in \mathcal{X},
\]
where $\mathcal{X} = \mathbb{R}^3 \times \Theta^h \times \Theta^v$ is defined as the state space domain consisting of the three-dimensional position and the camera orientation angles.  
Here, $[x_i, y_i, z_i]^\top$ is the position of the drone in $\Sigma_w$, and $\theta^h_i \in \Theta^h = [-\pi,\pi]$ and $\theta^v_i \in \Theta^v = (0,\pi/2]$ denote the yaw and pitch angles of the camera, respectively.  
The drone's motion is governed by the velocity input $u_i \in \mathcal{U} \subseteq \mathbb{R}^5$, where $\mathcal{U}$ represents the set of admissible inputs, such that
\[
\dot{p}_i = u_i.
\]
The drones are assumed to operate within a bounded region $\mathcal{P} \subset \mathbb{R}^3$. 
Note that in the previous work~\cite{lu2024angle}, only planar motion of the drones with fixed altitude was considered, resulting in a four-dimensional control input. 
In this work, we extend the formulation by explicitly controlling the drone altitude as well, resulting in a five-dimensional control input.

\begin{figure}[t]
    \centering
    \includegraphics[width=0.5\columnwidth]{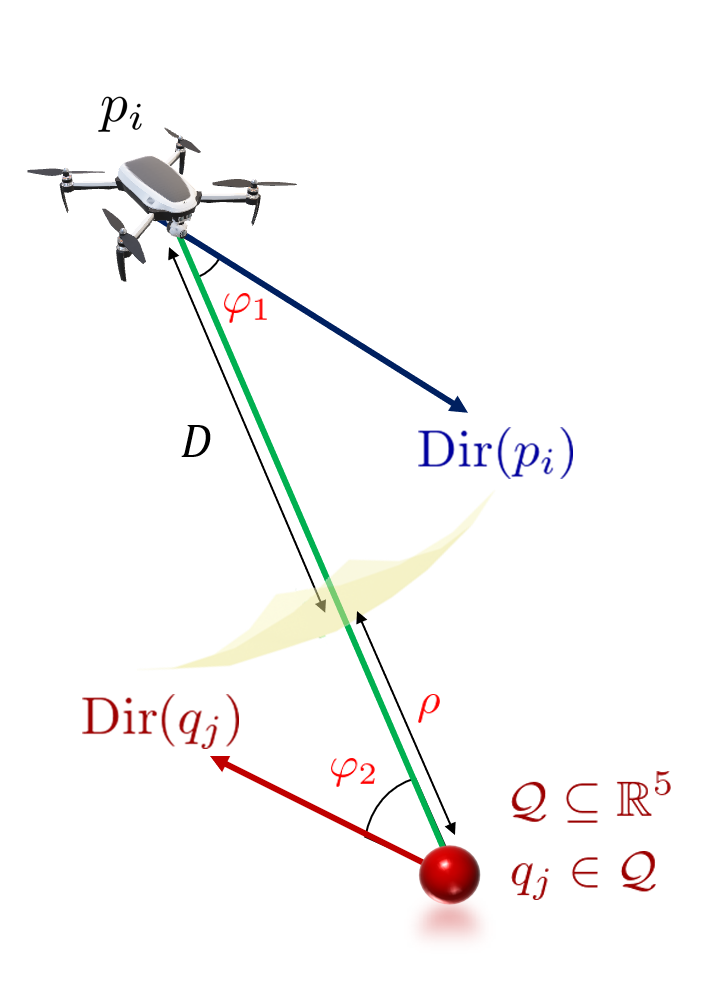}
    \caption{Visualization of the sensing performance metrics for a drone at state $p_i$ observing an observation point $q_j$, showing the distance $D$, deviation angles $\varphi_1$ and $\varphi_2$, and their associated direction vectors.}
    \label{fig:geo-ref}
\end{figure}

Next, we assume that the reconstruction target is contained within a region $\mathcal{B} \subset \mathbb{R}^3$, which defines the boundary enclosing the scene to be reconstructed (see Fig.~\ref{fig:prob_illust}).
For high-quality 3D reconstruction, each point in $\mathcal{B}$ should be observed from multiple viewpoints.  
To capture both spatial locations and viewing directions, we define a virtual field
\[
\mathcal{Q}_c = \mathcal{B} \times \Theta^h \times \Theta^v,
\]
forming a five-dimensional space, where each element represents a 3D point in $\mathcal{B}$ along with a pair of viewing angles.  
This space is discretized into $m$ cells, and the representative point of the $j$-th cell, referred to as the observation point, is given by
\[
q_j = [x_j, y_j, z_j, \theta^h_j, \theta^v_j]^\top \in \mathcal{Q}_c,
\]
with the collection of all such points denoted by $\mathcal{Q} = \{q_j \mid j \in \mathcal{M} = \{1,2,\ldots,m\}\}$. Fig.~\ref{fig:prob_illust} provides an illustration of the geometric relationship of the drone and observation points.

We introduce three operators that extract the position, angle, and direction components 
from the five-dimensional coordinates $p_i$ and $q_j$:  
\begin{align*}
\mathrm{Pos}([x,y,z,\theta^h,\theta^v]^\top) &= [x,y,z]^\top \in \mathbb{R}^3,\\
\mathrm{Ang}([x,y,z,\theta^h,\theta^v]^\top) &= [\theta^h,\theta^v]^\top \in \mathbb{R}^2,\\
\mathrm{Dir}([x,y,z,\theta^h,\theta^v]^\top) &=
\begin{bmatrix}
\cos\theta^h \cos\theta^v \\ 
\sin\theta^h \cos\theta^v \\ 
\sin\theta^v
\end{bmatrix} \in S^2,
\end{align*}
where $S^2 := \{\, v \in \mathbb{R}^3 \mid \|v\| = 1 \,\}$ denotes the two-dimensional unit sphere.  
For a drone state $p_i$, $\mathrm{Pos}(p_i)$ gives its spatial location, whereas for an observation point $q_j$, $\mathrm{Pos}(q_j)$ indicates the corresponding position in the virtual field. 
Similarly, $\mathrm{Ang}(p_i)$ represents the drone’s current camera orientation, while $\mathrm{Ang}(q_j)$ specifies the desired viewing direction. 
Finally, $\mathrm{Dir}(p_i)$ denotes the optical axis of the onboard camera in the world frame $\Sigma_w$, and $\mathrm{Dir}(q_j)$ gives the required viewing direction for proper observation.



\begin{remark}
In the present formulation, we consider only the control of the yaw and pitch angles of the drone camera. 
Although the roll angle can also be controlled, it is often omitted in practice since its effect on the sensing footprint is minimal \cite{liu2021aerial,smith2018aerial}. 
\end{remark}


\subsection{Sensing Performance}


To evaluate the quality of observation by the drones, we employ three metrics for a drone at state $p_i$ observing an observation point $q_j \in \mathcal{Q}$.  
These measures, illustrated in Fig.~\ref{fig:prob_illust}(b), capture the essential conditions for obtaining high-quality images required for 3D reconstruction:
(i) the alignment of the drone camera’s optical axis with the line-of-sight to the observation point,  
(ii) the orientation of the observation point relative to the drone, and  
(iii) the distance between the drone and the observation point.

The first measure evaluates the deviation between the drone camera’s optical axis and the line-of-sight vector from the drone to the observation point. This deviation is represented by the angle
\[
\varphi_1(p_i, q_j) = \arccos \left( 
\frac{\mathrm{Dir}(p_i) \cdot (\mathrm{Pos}(q_j) - \mathrm{Pos}(p_i))}
{\|\mathrm{Pos}(q_j) - \mathrm{Pos}(p_i)\|} \right),
\]
where a smaller $\varphi_1$ indicates that the camera is better aligned with the observation point.

The second measure evaluates whether the observation point is oriented toward the drone’s camera. This is quantified by the angle
\[
\varphi_2(p_i, q_j) = \arccos \left( 
-\frac{\mathrm{Dir}(q_j) \cdot (\mathrm{Pos}(q_j) - \mathrm{Pos}(p_i))}
{\|\mathrm{Pos}(q_j) - \mathrm{Pos}(p_i)\|} \right),
\]
where a smaller $\varphi_2$ indicates that the target surface is directly visible to the camera, allowing clear observation of its features.

The third measure accounts for the distance between the drone and the observation point.  
To ensure proper image resolution, the drone should maintain an optimal observation distance $D$, which depends on the intrinsic properties of the camera such as focal length and sensor resolution.  
The deviation from this ideal distance is defined as
\[
\rho(p_i, q_j) = \|\mathrm{Pos}(q_j) - \mathrm{Pos}(p_i)\| - D.
\]

These three measures collectively characterize the sensing geometry. When $\varphi_1$, $\varphi_2$, and $\rho$ are all close to zero, the sensing performance on the point $q_j$ is maximized.
To combine the three measures into a single performance metric, we require a function that is bounded between 0 and 1 to represent a normalized sensing performance level.
In addition, the function should have a smooth, non-zero gradient to ensure the existence of feasible control inputs that can improve the sensing performance. 
The Gaussian function naturally satisfies these properties and is therefore adopted in this work.
%
The resulting sensing performance function is defined as
\begin{equation}
    \begin{split}
        h_1(p_i, q_j) = &
        \exp\!\left(-\frac{(1-\cos \varphi_1)^2}{2\sigma_1^2}\right)
        \exp\!\left(-\frac{(1-\cos \varphi_2)^2}{2\sigma_2^2}\right) \\
        &\times \exp\!\left(-\frac{\rho^2}{2\sigma_3^2}\right),
    \end{split}
    \label{eq:sensing_performance}
\end{equation}
where $\sigma_1$, $\sigma_2$, and $\sigma_3$ are sensitivity parameters that determine the tolerance to deviations for each of the three sensing measures.  
These parameters should be tuned to reflect the characteristics of the drone’s camera and the desired sensing precision.


The parameter $\sigma_1$ is tuned according to the camera’s field of view so that the performance function reflects the angular characteristics of the camera. A smaller $\sigma_1$ enforces stricter alignment, while a larger $\sigma_1$ allows more tolerance.

The parameter $\sigma_2$ relates to the angular density of observation points and is set to ensure smooth overlap between neighboring points without gaps or redundancy. For evenly spaced observation points, $\sigma_2$ is chosen such that the performance level at the midpoint between two adjacent points is $0.5$, leading to
\[
\sigma_2 = \sqrt{\frac{(1 - \cos (\Delta \theta / 2))^2}{2 \log 2}},
\]
where $\Delta \theta$ denotes the angular spacing.

The parameter $\sigma_3$ governs acceptable deviation from the desired observation distance $D$. It is determined by camera properties: high-resolution cameras require smaller $\sigma_3$ for strict distance enforcement, while lower-resolution cameras allow larger $\sigma_3$ for more flexibility.

By carefully tuning these parameters, the function $h_1(p_i, q_j)$ provides a smooth quantitative measure of how effectively a drone at state $p_i$ observes an observation point $q_j$.  
This performance metric will serve as the foundation for updating the importance index, which will be explained below.

\subsection{Importance Index and Global Objective}
\label{sec:importance_index}

We next define an importance index $\phi_j(t) \in [0, \infty)$ for each observation point $q_j \in \mathcal{Q}$, which represents the priority of observing that point.  
At the beginning of the mission, all points are initialized with the same value, $\phi_j(0) = 1$, indicating that no region has been observed yet.  
As the drones capture images over time, $\phi_j$ decreases according to the quality of the observations provided by the drones.  
The mission is considered complete when all $\phi_j$ values approach zero, meaning that every point in the region has been sufficiently observed.

The evolution of $\phi_j$ is governed by the sensing performance function $h_1(p_i, q_j)$ introduced in \eqref{eq:sensing_performance}.  
For each point $q_j$, its update rule is given by
\begin{equation}
    \dot{\phi}_j = -\delta_1 \max_{i \in \mathcal{I}} h_1(p_i, q_j) \, \phi_j,
    \label{eq:phi_j_update_h1}
\end{equation}
where $\delta_1 > 0$ is a scaling parameter and $\mathcal{I}$ is the set of all drones.  
This rule ensures that the decrease rate of $\phi_j$ depends on both its current value and the maximum sensing performance among all drones observing that point.  
If a point is being observed with high quality by at least one drone, $\phi_j$ decreases rapidly, while poorly observed points retain higher values, encouraging the drones to focus on them.

In this context, the control goal is to minimize the overall importance index across the entire virtual field by optimizing the drones' trajectories.  
The global objective function $J$, representing the average importance level over all observation points, is defined as
\begin{equation}
    J = \sum_{j \in \mathcal{M}} \frac{\phi_j}{|\mathcal{M}|} = \sum_{j=1}^{m} \frac{\phi_j}{m},
    \label{eqn:j}
\end{equation}
where $\mathcal{M}$ is the set of all discretized observation points, and $m = |\mathcal{M}|$ is its cardinality.  
Minimizing $J$ drives the drones to collectively focus on areas with high importance, ensuring that the entire region is sufficiently covered.  

Note that the parameter $\delta_1$ specifies how densely images associated with $q_j$ should be sampled to reconstruct the structure around $q_j$.
However, the required image density depends on the structural complexity around $\mathrm{Pos}(q_j)$, which is unknown \emph{a priori}.
Thus, systematically tuning $\delta_1$ is inherently challenging, and a similar limitation is observed in path planning approaches~\cite{tmuvsic2020current, torres2016coverage, bircher2016receding, delmerico2018comparison, song2017online, roberts2017submodular}.
To address this issue, in the next section we will extend this formulation by incorporating online map feedback, allowing the importance index to adapt dynamically to previously unknown scene structures. 

\section{Online Map Feedback for Coordinated Image Sampling}
\label{sec:map_feedback}

We adopt NeuralRecon~\cite{sun2021neuralrecon} as the online 3D reconstruction module.  
It takes as input a sequence of monocular RGB images and their corresponding camera poses, and outputs a triangular mesh representation of the scene, as illustrated in Fig.~\ref{fig:neural_recon_pipeline}.  
The mesh is updated at discrete events $t \in \mathcal{T} = \{1,2,\ldots\}$ and is denoted by $(V^t, F^t)$, where $V^t$ is the set of vertices and $F^t$ is the set of triangular faces.  
This mesh output will be used as the basis for the online map feedback mechanism in this work, which will be described in the following subsection.  
For further details on NeuralRecon, please refer to~\cite{sun2021neuralrecon}.




\begin{figure}[t]
    \centering
    \includegraphics[width=0.98\columnwidth]{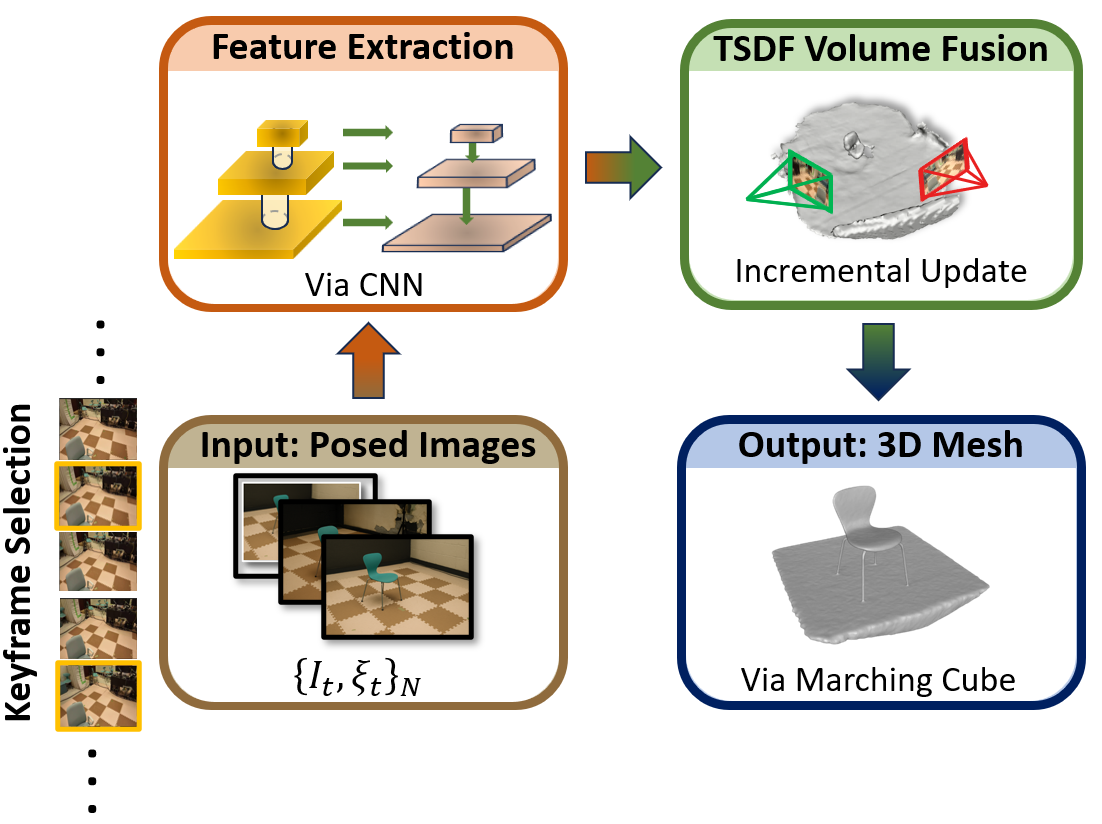}
    \caption{The NeuralRecon pipeline~\cite{sun2021neuralrecon} takes posed monocular RGB keyframes ${(I_t,\xi_t)}$ as input and outputs a triangular mesh $(V^t,F^t)$, updated at discrete events $t \in \mathcal{T}$.}
    \label{fig:neural_recon_pipeline}
\end{figure}

\subsection{Sampling Feedback from Mesh Evolution}
\label{sec:mesh_feedback}

We are interested in exploiting feedback from online reconstruction to guide image sampling during flight, with the aim of improving the quality of the reconstructed map. If ground truth were available, one could simply identify poorly reconstructed regions and prioritize them for additional sampling by the drones. In realistic scenarios, however, no prior knowledge of the scene is available, which raises the challenge of assessing reconstruction quality using only the evolving map itself~\cite{zhang2021metric,coupry2025assessing}. 

To address this challenge, we propose to use mesh changes between successive reconstructions as a feedback signal for guiding the sampling process. Intuitively, if a certain region of the scene exhibits large mesh changes after being sampled by the drones, it suggests that the region is still evolving and therefore may require further observations. Conversely, if the mesh in a region stabilizes with only small changes, it is likely that the region has been sufficiently reconstructed and no longer needs additional sampling. Building on this idea, the degree of mesh change is used as a feedback to update the importance indices $\phi_j$ in the previous angle-aware coverage formulation. 
If a region around point $q_j$ continues to exhibit large mesh changes even after being sampled, $\phi_j$ is increased, prompting the drones to revisit and collect additional observations. This process continues until the mesh changes diminish, indicating that the reconstruction has stabilized and further sampling is no longer necessary.

To formalize this concept, we introduce a nonnegative feedback function $h_2$ that quantifies the degree of mesh change as a function of both location $q_j$ and reconstruction update event $t$,  
\[
h_2(q_j,t):\ \mathcal{Q} \times \mathcal{T} \to \mathbb{R}_{\ge 0}.
\]  
Here, large values of $h_2(q_j,t)$ correspond to regions with significant mesh changes. The detailed formulation of $h_2$ will be discussed in the next subsection, as it depends on the specific method used to quantify the mesh changes.  

We use this feedback function to update the importance index $\phi_j$ at each reconstruction update event. As these events occur at discrete time instances, the update is modeled as a jump in $\phi_j$ and expressed as
\begin{equation}
    \phi_j(t^+) = \phi_j(t^-) + \delta_2\, h_2(q_j,t),
    \label{eq:phi_j_update_h2}
\end{equation}
where $\delta_2 > 0$ is a gain parameter that determines the strength of the feedback effect. Here, $t^-$ and $t^+$ denote the values of $\phi_j$ immediately before and after the update event, respectively.

\subsection{Mesh-Change Quantification}
\label{sec:mesh_change_quantification}

We now elaborate the feedback function $h_2(q_j,t)$ by defining how to quantify changes in the reconstructed mesh between consecutive updates. At each update event $t \in \mathcal{T}$, the mesh is represented as a set of vertices
\[
V^t = \{v_\ell^t \in \mathbb{R}^3 \mid \ell = 1, \ldots, N^t\} \subseteq \mathbb{R}^3,
\]
where $v_\ell^t$ is the position of the $\ell$-th vertex and $N^t$ is the total number of vertices at time $t$. As new images are integrated, both the positions and number of vertices evolve, reflecting local changes in the reconstruction.

Our objective is to produce a nonnegative function over virtual points $q_j \in \mathcal{Q}$ such that regions experiencing larger local mesh changes yield larger values of $h_2(q_j,t)$. To achieve this, we propose two complementary approaches for quantifying mesh change:  
(1) a \textbf{3D Grid-Based Method}, which can be computationally efficient depending on the defined sparsity of the grid, and  
(2) an \textbf{M3C2-Based Method}, which directly captures geometric displacement between consecutive meshes.

\vspace{0.5\baselineskip}

\subsubsection{3D Grid-Based Method}
\label{sec:grid_metric}

The first method quantifies mesh changes by discretizing the reconstruction space into a set of uniformly spaced evaluation points:
\[
\mathcal{R} := \{r_k \in \mathbb{R}^3 \mid k = 1, \ldots, s\},
\]
where each $r_k$ represents a fixed position in 3D space. By selecting the number of grid points $s$ smaller than the number of virtual points $|\mathcal{Q}|$, the computational load can be reduced while still capturing the essential mesh changes.

For each evaluation point $r_k$, the set of mesh vertices that are closest to $r_k$ at update time $t$ is defined as
\[
V_k^t := \Bigl\{\ell \in \{1,\ldots,N^t\} \ \Big|\ r_k = \arg\min_{r \in \mathcal{R}} \|r - v_\ell^t\|, \ v_\ell^t \in \mathcal{B} \Bigr\},
\]
where $\mathcal{B}$ denotes the bounded reconstruction region defined earlier in Section~\ref{subsec:2a} (see Fig.~\ref{fig:prob_illust}a).  
The local mesh change at $r_k$ between two consecutive updates $t-1$ and $t$ is then given by the difference in the number of associated vertices:
\begin{equation}
    \label{eq:delta_mesh}
    \Delta V^t(r_k) := \big|V_k^t\big| - \big|V_k^{t-1}\big|.
\end{equation}

To integrate these changes into the feedback function, we define $h_2(q_j,t)$ as
\begin{equation}
    \label{eq:h2_grid}
    \begin{split}
        h_2(q_j,t) = \sum_{r_k \in \mathcal{R}} 
        &\tanh^2\!\Biggl(\Bigl(\frac{\Delta V^t(r_k)}{\kappa}\Bigr)^2\Biggr) \\
        &\times \exp\!\Biggl(-\frac{\|\mathrm{Pos}(q_j)-r_k\|^2}{2\sigma_4^2}\Biggr),
    \end{split}
\end{equation}
where $\kappa > 0$ normalizes the sensitivity to mesh changes, and $\sigma_4 > 0$ determines the spatial influence range of each grid point. The $\tanh^2(\cdot)$ term normalizes the raw mesh change magnitude to the range $[0,1)$, effectively suppressing the effect of small fluctuations while emphasizing regions with significant local changes. From an implementation perspective, the Gaussian weights $\exp(-\|\mathrm{Pos}(q_j)-r_k\|^2/2\sigma_4^2)$ depend only on the positions of $q_j$ and $r_k$ and can thus be precomputed to reduce runtime cost. 

\vspace{0.5\baselineskip}

\subsubsection{M3C2-Based Method}
\label{sec:m3c2_metric}

While the grid-based method provides a lightweight way to approximate mesh changes, it only captures changes in local vertex density and does not directly measure geometric displacement between two consecutive meshes. To address this limitation, and for comparison, we employ the Multiscale Model-to-Model Cloud Comparison (M3C2) distance~\cite{lague2013accurate}, which directly computes local surface-to-surface displacements by comparing two point clouds.

In this approach, the mesh vertices $V^{t-1}$ and $V^t$ are treated as two consecutive point clouds. A set of uniformly spaced core points $\mathcal{R} = \{r_k \in \mathbb{R}^3 \mid k = 1, \ldots, s \}$ is defined, and the M3C2 distance is evaluated at these core points. For each $r_k$, the local displacement between the two meshes is computed and denoted as
\[
L_{\mathrm{M3C2}}(r_k,t) \geq 0,
\]
which represents the average perpendicular distance between the two point clouds in the local neighborhood of $r_k$.  

The resulting feedback function is then constructed as
\begin{equation}
\label{eq:h2_m3c2}
    h_2(q_j,t) = \sum_{r_k \in \mathcal{R}} L_{\mathrm{M3C2}}(r_k,t)
    \exp\!\Biggl(-\frac{\|\mathrm{Pos}(q_j)-r_k\|^2}{2\sigma_4^2}\Biggr),
\end{equation}
where the exponential term acts as a Gaussian kernel to localize the influence of each core point.  
Unlike the grid-based method in \eqref{eq:h2_grid}, no additional normalization such as $\tanh^2(\cdot)$ is required here because $L_{\mathrm{M3C2}}$ directly reflects the magnitude of geometric displacement.  

The M3C2-based method provides a more accurate measure of mesh evolution by capturing true geometric changes. However, this comes with a significantly higher computational cost, as the distances between two full point clouds must be recomputed at every update. In this work, we mainly use the M3C2-based method as a reference to compare and validate the performance of the proposed 3D grid-based method.

\section{Controller Design with Online Map Feedback}
\label{sec:controller}

In this section, we design a controller for coordinated image sampling with multiple drones using a QP-based approach~\cite{ames2016control}. 
The controller builds on the angle-aware coverage formulation introduced in Section~\ref{sec:problem_setting} and incorporates the online map feedback mechanism described in Section~\ref{sec:map_feedback}. 
We begin by describing the control objective and the necessary constraints, followed by the complete QP formulation and the map feedback activation mechanism.

\subsection{Control Objective and Constraints}
\label{sec:objective_constraints}

The control goal is to drive the drones toward complete coverage of the scene by collecting images across all regions, which corresponds to reducing $J$, defined in \eqref{eqn:j}, as close to zero as possible. However, simply reducing $J$ is not sufficient to ensure efficient progress, since even a random-walk behavior could eventually decrease $J$ over time without systematic coordination.  
To achieve reliable and structured progress, we introduce a constraint on the rate of decrease of $J$, similarly to~\cite{lu2024angle, shimizu2021angle}.

\textbf{Sampling performance constraint:}  
To guarantee a desired level of sampling performance, we require the objective $J$ to decrease at a prescribed rate:
\begin{equation}
    \dot{J} \leq -\gamma, \qquad \forall \, \tau \notin \mathcal{T}_c,
    \label{eq:j_dot_constraint}
\end{equation}
where $\gamma > 0$ is a design parameter, and $\mathcal{T}_c$ is the set of continuous time intervals between discrete reconstruction updates $\mathcal{T}$.  
This ensures that $J$ decreases according to the desired performance level between two successive map-feedback updates.

Let us now evaluate the time derivative of $J$ for any $\tau \notin \mathcal{T}_c$:
\begin{align}
    \dot J
    &= \frac{1}{|\mathcal{M}|} \sum_{j\in\mathcal{M}} \dot{\phi}_j
     = -\frac{1}{|\mathcal{M}|} \sum_{j\in\mathcal{M}} \delta_1 \max_{i\in\mathcal{I}} h_1(p_i,q_j)\,\phi_j \nonumber\\
    &= -\frac{1}{|\mathcal{M}|} \sum_{i=1}^{n} \sum_{j\in\mathcal{V}_i(p)} \delta_1 h_1(p_i,q_j)\,\phi_j
     = -\frac{1}{|\mathcal{M}|} \sum_{i=1}^{n} I_i, \label{eq:Jdot}
\end{align}
where  
\[
\mathcal{V}_i(p) := \big\{ j \in \mathcal{M} \mid h_1(p_i,q_j) \leq h_1(p_k,q_j),\ \forall k \in \mathcal{I} \big\}
\]
is a Voronoi-like partition induced by the sensing function $h_1$, and
\begin{equation}
    I_i := \sum_{j\in\mathcal{V}_i(p)} \delta_1 h_1(p_i,q_j)\,\phi_j
    \label{eq:I_i}
\end{equation}
represents the contribution of drone $i$ in decreasing $J$.

\begin{theorem}
    Suppose that each drone $i$ satisfies the local condition
    \begin{equation}
        \dot{J}_i := -\frac{I_i}{|\mathcal{V}_i(p)|} \leq -\gamma,
        \label{eq:j_i_dot}
    \end{equation}
    for all $i \in \mathcal{I}$. Then, the global constraint \eqref{eq:j_dot_constraint} is also satisfied.
\end{theorem}

\begin{proof}
    Multiplying \eqref{eq:j_i_dot} by $|\mathcal{V}_i(p)|$ gives $-I_i \leq -|\mathcal{V}_i(p)|\gamma$.  
    Summing over all drones,
    \[
    -\sum_{i=1}^{n} I_i \leq -\sum_{i=1}^{n} |\mathcal{V}_i(p)| \gamma = -|\mathcal{M}| \gamma.
    \]
    Substituting into \eqref{eq:Jdot} yields $\dot{J} \leq -\gamma$, completing the proof.
\end{proof}

To enforce \eqref{eq:j_i_dot}, we define the function
\begin{equation}
    b_{i,I} = I_i - |\mathcal{V}_i(p)| \gamma,
    \label{eq:b_i,I}
\end{equation}
where $b_{i,I} \geq 0$ is equivalent to satisfying the local sampling performance constraint.  
This condition is satisfied when the velocity input $u_i$ is chosen so that the following affine inequality holds:
\begin{equation}
    \frac{\partial b_{i,I}}{\partial p_i}^\top u_i + \frac{\partial b_{i,I}}{\partial \phi_j} \dot{\phi}_j +
    \alpha_1(b_{i,I}) \geq 0,
    \label{eq:cbf_sampling}
\end{equation}
where $\alpha_1(\cdot)$ is a locally Lipschitz extended class-$\mathcal{K}$ function~\cite{ames2016control}.  

\textbf{Gimbal pitch limitation constraint:}  
Following \cite{lu2024angle}, we impose a constraint to ensure the gimbal pitch angle remains within its hardware limits, bounded by $\theta^v_{\min}$ and $\theta^v_{\max}$.  
The corresponding control barrier function is defined as
\begin{equation}
    b_{i,\theta^v} = (\theta^v_{\max} - \theta^v_{\min})^2 -
    \left(\theta^v_i - \frac{\theta^v_{\min} + \theta^v_{\max}}{2}\right)^2 \geq 0,
    \label{eq:cbf_gimbal_pitch}
\end{equation}
where $b_{i,\theta^v} \geq 0$ enforces the allowable range of motion.

This condition is satisfied when the control input $u_i$ is chosen to meet the following affine inequality:
\begin{equation}
    \frac{\partial b_{i,\theta^v}}{\partial p_i}^\top u_i + \alpha_2(b_{i,\theta^v}) \geq 0,
    \label{eq:cbf_gimbal_pitch_u}
\end{equation}
where $\alpha_2(\cdot)$ is a locally Lipschitz extended class-$\mathcal{K}$ function.  

\textbf{Collision avoidance constraint:}  
To ensure safe operation of the multi-drone system, we also introduce a constraint that enforces a minimum separation distance $d > 0$ between drones \cite{shimizu2021angle}.  
The corresponding control barrier function is given by
\begin{equation}
    b_{i,\mathrm{ca}} = \min_{j \in \mathcal{I} \setminus \{i\}}
    \big\| \mathrm{Pos}(p_i) - \mathrm{Pos}(p_j) \big\|^2 - d^2 \geq 0,
    \label{eq:b_i,ca}
\end{equation}
where $b_{i,\mathrm{ca}} \geq 0$ guarantees that the distance between any two drones stays above the safety threshold $d$.

This condition is satisfied when the control input $u_i$ is chosen to meet the following affine inequality:
\begin{equation}
    \frac{\partial b_{i,\mathrm{ca}}}{\partial p_i}^\top u_i + \alpha_3(b_{i,\mathrm{ca}}) \geq 0,
    \label{eq:cbf_collision_u}
\end{equation}
where $\alpha_3(\cdot)$ is a locally Lipschitz extended class-$\mathcal{K}$ function.




\begin{remark}
    In practice, additional constraints may be required depending on the specific application, such as flight region boundaries or other operational limits.  
    These can be incorporated into the framework using the same control barrier function approach described above.
\end{remark}

\subsection{QP-based Controller}
\label{sec:qp_controller_sec}

We are now ready to present the QP-based controller that integrates all the CBF constraints introduced in the previous subsection, namely the sampling performance constraint \eqref{eq:cbf_sampling}, the gimbal pitch limitation constraint \eqref{eq:cbf_gimbal_pitch_u}, and the collision avoidance constraint \eqref{eq:cbf_collision_u}.
These constraints are combined into a single quadratic program that computes the velocity input $u_i$ for each drone by solving
\begin{subequations}
\label{eq:qp_controller}
    \begin{align}
        (u_i^{*},\,w_i^{*}) \;&=\;
        \arg\min_{(u_i,w_i)\in \mathcal{U}\times\mathbb{R}} \ \epsilon \,\|u_i\|^2 + |w_i|^2 \\
        \text{s.t.}\quad
        & \left(\frac{\partial b_{i,I}}{\partial p_i}\right)^{\!\top} u_i
          + \left(\frac{\partial b_{i,I}}{\partial \phi_j}\right)\dot\phi_j
          +  \alpha_1(b_{i,I}) \ \ge\ w_i, \label{eq:qp_controller_progress}\\
        & \left(\frac{\partial b_{i,\theta^v}}{\partial p_i}\right)^{\!\top} u_i
          + \alpha_2(b_{i,\theta^v}) \ \ge\ 0, \label{eq:qp_controller_pitch}\\
        & \left(\frac{\partial b_{i,\mathrm{ca}}}{\partial p_i}\right)^{\!\top} u_i
          + \alpha_3(b_{i,\mathrm{ca}}) \ \ge\ 0, \label{eq:qp_controller_collision}
    \end{align}
\end{subequations}
where $\epsilon>0$ weights the control effort, and $w_i$ is a slack variable that allows relaxation of the sampling performance constraint \eqref{eq:qp_controller_progress}, while keeping the safety constraints \eqref{eq:qp_controller_pitch} and \eqref{eq:qp_controller_collision} strict.  

\begin{theorem}\label{th:qp_form}
    Suppose that no $q_j$ (for $j \in \mathcal{M}$) is located on the boundary of $\mathcal{V}_i(p)$. 
    Let $\alpha_1(b_{i,I})=a_1 b_{i,I}$, $\alpha_2(b_{i,\theta^v})=a_2 b_{i,\theta^v}$, and $\alpha_3(b_{i,\mathrm{ca}})=a_3 b_{i,\mathrm{ca}}$, where $a_1, a_2, a_3 \in \mathbb{R}_{>0}$.
    Then the problem in \eqref{eq:qp_controller} is equivalently reformulated as
    \begin{subequations}\label{eq:qp_controller_final}
        \begin{align}
            (u_i^{*},w_i^{*}) 
            &= \argmin_{(u_i,w_i)\in\mathcal{U}\times\mathbb{R}} \ \epsilon\,\|u_i\|^2 + |w_i|^2 \\[4pt]
            \text{s.t.}\quad
            & \xi_{1i}^{\top} u_i + \xi_{2i} \ \ge\ w_i, \\
            & \chi_{1i}^{\top} u_i + \chi_{2i} \ \ge\ 0, \\
            & \rho_{1i}^{\top} u_i + \rho_{2i} \ \ge\ 0,
        \end{align}
    \end{subequations}
    where
    \begin{subequations}\label{eq:xi-defs}
        \begin{align}
            \xi_{1i} &= \delta_1 \sum_{j \in \mathcal{V}_i(p)} 
            \frac{\partial h_1(p_i,q_j)}{\partial p_i}\, \phi_j, \\[6pt]
            \xi_{2i} &= -a_1\,|\mathcal{V}_i(p)|\,\gamma \notag\\
                     &\quad + \sum_{j \in \mathcal{V}_i(p)} 
                     \delta_1 h_1(p_i,q_j)\,\phi_j
                     \big(a_1 - \delta_1 h_1(p_i,q_j)\big),
        \end{align}
    \end{subequations}
    \begin{subequations}\label{eq:chi-defs}
        \begin{align}
            \chi_{1i} &= 
            \big[\,0\ \ 0\ \ 0\ \ 0\ \ 2(\theta^v_i - \tfrac{\theta^v_{\min}+\theta^v_{\max}}{2})\,\big]^{\!\top}, \\[6pt]
            \chi_{2i} &= 
            a_2\!\left((\theta^v_{\max}-\theta^v_{\min})^2 
            - \left(\theta^v_i - \tfrac{\theta^v_{\min}+\theta^v_{\max}}{2}\right)^2\right),
        \end{align}
    \end{subequations}
    \begin{subequations}\label{eq:rho-defs}
        \begin{align}
            \rho_{1i} &= 
            \big[\,2\big(\mathrm{Pos}(p_i)-\mathrm{Pos}(p_{j^\star})\big)^{\!\top}\ \ 0\ \ 0\,\big]^{\!\top}, \\[6pt]
            \rho_{2i} &= 
            a_3\!\left(\big\|\mathrm{Pos}(p_i)-\mathrm{Pos}(p_{j^\star})\big\|^2 - d^2\right),
        \end{align}
    \end{subequations}
    
    \noindent
    where $j^\star = \arg\min_{j \in \mathcal{I} \setminus \{i\}} 
    \|\mathrm{Pos}(p_i)-\mathrm{Pos}(p_j)\|^2$ denotes the closest neighbor of drone $i$.  
\end{theorem}

\begin{proof}
    The constraints in \eqref{eq:qp_controller} involve the terms 
    $\frac{\partial b_{i,I}}{\partial p_i}$, $\frac{\partial b_{i,I}}{\partial \phi_j}$, $\alpha_1(b_{i,I})$, 
    $\frac{\partial b_{i,\theta^v}}{\partial p_i}$, $\alpha_2(b_{i,\theta^v})$, 
    $\frac{\partial b_{i,\mathrm{ca}}}{\partial p_i}$, and $\alpha_3(b_{i,\mathrm{ca}})$.  
    The quantities $\xi_{1i}$, $\xi_{2i}$, $\chi_{1i}$, $\chi_{2i}$, $\rho_{1i}$, and $\rho_{2i}$ are derived as follows.
    
    The function $b_{i,I}$ in \eqref{eq:b_i,I} with $I_i$ from \eqref{eq:I_i} gives
    \begin{align*}
        \xi_{1i} &= \frac{\partial b_{i,I}}{\partial p_i} 
        = \delta_1 \sum_{j \in \mathcal{V}_i(p)} 
           \frac{\partial h_1(p_i,q_j)}{\partial p_i}\,\phi_j.
    \end{align*}
    Using \eqref{eq:phi_j_update_h1} and $\alpha_1(b_{i,I}) = a_1 b_{i,I}$, we have
    \begin{align*}
        \xi_{2i} 
        &= \left(\frac{\partial b_{i,I}}{\partial \phi_j}\right)\dot{\phi}_j + \alpha_1(b_{i,I}) = - a_1\,|\mathcal{V}_i(p)|\,\gamma 
        +\\[4pt]
        & \sum_{j \in \mathcal{V}_i(p)} 
           \delta_1 h_1(p_i,q_j)\,\phi_j
           \big(a_1 - \delta_1 h_1(p_i,q_j)\big).
    \end{align*}
    Thus,
    \[
        \left(\frac{\partial b_{i,I}}{\partial p_i}\right)^{\!\top}u_i 
        + \left(\frac{\partial b_{i,I}}{\partial \phi_j}\right)\dot{\phi}_j 
        + \alpha_1(b_{i,I})
        = \xi_{1i}^{\top} u_i + \xi_{2i}.
    \]

    From \eqref{eq:cbf_gimbal_pitch},
    \[
        \chi_{1i} = \frac{\partial b_{i,\theta^v}}{\partial p_i}
        = \big[\,0\ \ 0\ \ 0\ \ 0\ \ 2(\theta^v_i - \tfrac{\theta^v_{\min}+\theta^v_{\max}}{2})\,\big]^{\!\top},
    \]
    and substituting $\alpha_2(b_{i,\theta^v}) = a_2 b_{i,\theta^v}$ gives $\chi_{2i}$.  
    Hence,
    \[
        \left(\frac{\partial b_{i,\theta^v}}{\partial p_i}\right)^{\!\top}u_i + \alpha_2(b_{i,\theta^v})
        =\chi_{1i}^{\top} u_i + \chi_{2i}.
    \]

    The collision-avoidance function is defined in \eqref{eq:b_i,ca}.  
    Let $j^\star = \arg\min_{j \neq i} \|\mathrm{Pos}(p_i)-\mathrm{Pos}(p_j)\|^2$ be the closest neighbor.  
    Its derivative is
    \[
        \rho_{1i} = \frac{\partial b_{i,\mathrm{ca}}}{\partial p_i}
        = \big[\,2(\mathrm{Pos}(p_i)-\mathrm{Pos}(p_{j^\star}))^{\!\top}\ \ 0\ \ 0\,\big]^{\!\top},
    \]
    and substituting $\alpha_3(b_{i,\mathrm{ca}}) = a_3 b_{i,\mathrm{ca}}$ yields
    \[
        \rho_{2i} = a_3\big(\|\mathrm{Pos}(p_i)-\mathrm{Pos}(p_{j^\star})\|^2 - d^2\big).
    \]
    Thus,
    \[
        \left(\frac{\partial b_{i,\mathrm{ca}}}{\partial p_i}\right)^{\!\top}u_i + \alpha_3(b_{i,\mathrm{ca}})
        = \rho_{1i}^{\top} u_i + \rho_{2i}.
    \]

    \vspace{4pt}
    Collecting these inequalities yields the QP form
    \eqref{eq:qp_controller_final}, proving its equivalence to the original controller
    \eqref{eq:qp_controller}.
\end{proof}



\subsection{Map Feedback Activation}
\label{sec:map_feedback_activation}

In practice, we would like the drones to initially behave in an exploratory manner, scanning the entire field to gather rough information about the environment and generate a coarse 3D representation of the scene.  
As the mission progresses, their behavior should gradually shift toward refinement, where sampling effort is concentrated on regions requiring further improvement in reconstruction quality.  

To realize this behavior, we introduce a threshold $J_{\text{th}}$ on the global objective function $J$.  
When $J \geq J_{\text{th}}$, the importance index $\phi_j$ is updated only based on the drones' sensing performance \eqref{eq:phi_j_update_h1}.  
Once $J < J_{\text{th}}$, the map feedback is activated, and $\phi_j$ is additionally updated using feedback derived from mesh changes \eqref{eq:phi_j_update_h2}, encouraging the drones to revisit regions where the reconstruction is still evolving.


The switching rule for the map feedback gain $\delta_2$ is given by
\begin{equation}
\label{eq:delta2_switch}
    \delta_2 =
    \begin{cases}
    0 & \text{if } J \geq J_{\text{th}},\\[2pt]
    \delta_2' & \text{if } J < J_{\text{th}},
    \end{cases}
\end{equation}
where $\delta_2'>0$ determines the strength of the map feedback when activated. This mechanism allows the controller to smoothly shift its behavior from initial exploration to later refinement as coverage progresses.  
At present, we have not investigated a detailed theoretical guideline for selecting the threshold value $J_{\text{th}}$.  
Instead, in this work, we treat $J_{\text{th}}$ as a tunable design parameter and will provide empirical observations in Section~\ref{sec:sim_results} to illustrate the effects of different choices of $J_{\text{th}}$. The complete control algorithm is summarized in Algorithm~\ref{alg:overall}.


\begin{figure*}[t]
    \centering
    \includegraphics[width=0.75\textwidth]{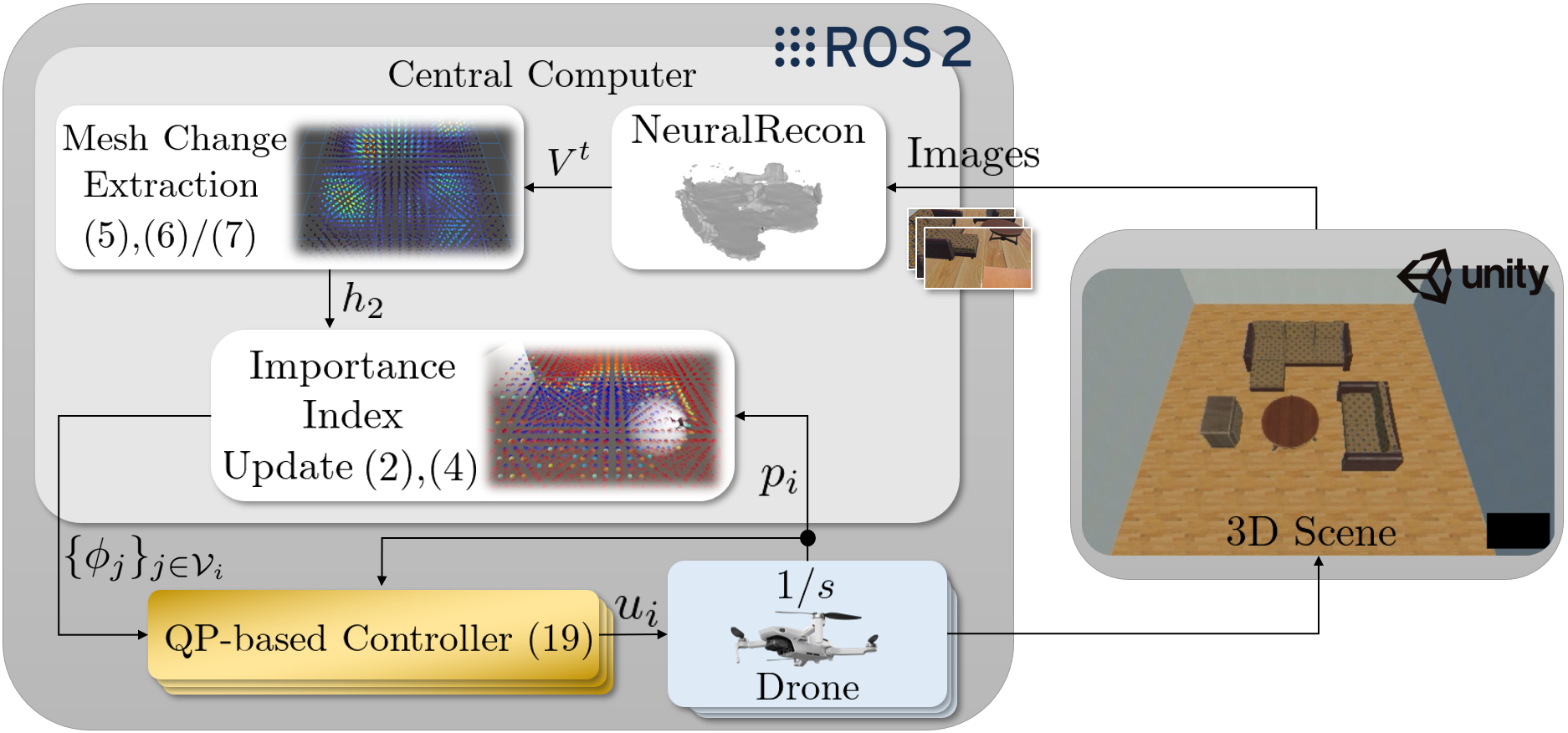}
    \caption{System architecture of Coverage-Recon. 
    The left block shows the controller framework implemented in ROS2, which is used in both simulation and real-world experiments. 
    In simulation, the real environment is replaced by a virtual 3D scene created in Unity (right block), which serves as the target to be reconstructed by the drone.}
    \label{fig:system_architecture}
\end{figure*}

\section{Simulation Results}
\label{sec:sim_results}

In this section, we verify the proposed Coverage-Recon algorithm through simulation using ROS2 middleware~\cite{macenski2022robot} and Unity for photorealistic rendering~\cite{juliani2018unity}. 
The results are 
presented in two parts: 
(i) a single-drone evaluation, which validates the impact of the proposed online map feedback mechanism by comparing the generated maps with baseline methods, both qualitatively and quantitatively; and  
(ii) a multi-drone evaluation, which demonstrates the scalability of the framework and its ability to complete missions more efficiently than single-drone operation.

\begin{figure*}[t]
    \centering
    \includegraphics[width=\textwidth]{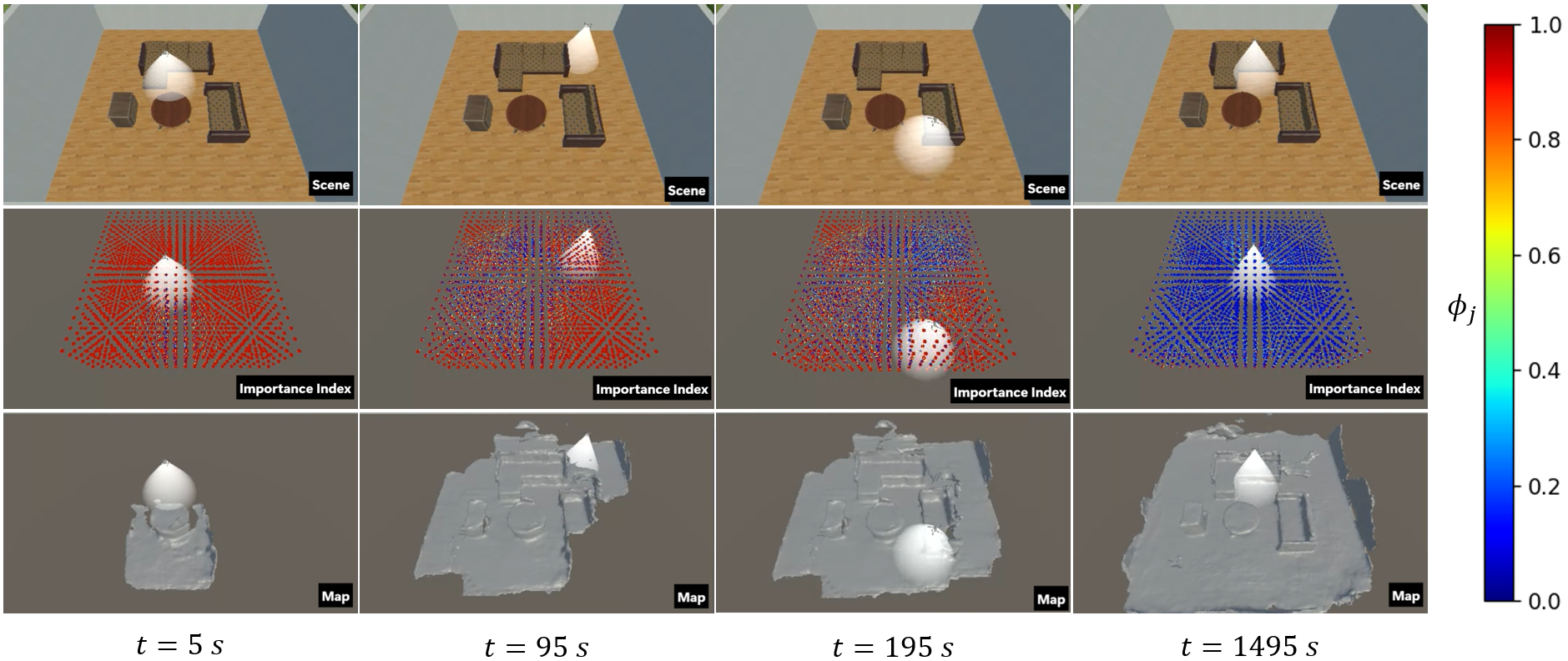}
    \caption{Simulation snapshots of the proposed Coverage-Recon framework with a single drone, using the 3D Grid method for mesh-change quantification and a threshold $J_{\text{th}} = 50\%$. 
    The top row shows the simulated scene, the middle row the importance index $\phi_j$ (red: high, blue: low), and the bottom row the reconstructed map. 
    A video of the simulation can be viewed \href{https://www.youtube.com/watch?v=FrcXFXxOk8A}{here}.}
    \label{fig:sim_snapshot}
\end{figure*}

\begin{figure}[t]
    \centering
    \includegraphics[width=1.0\columnwidth]{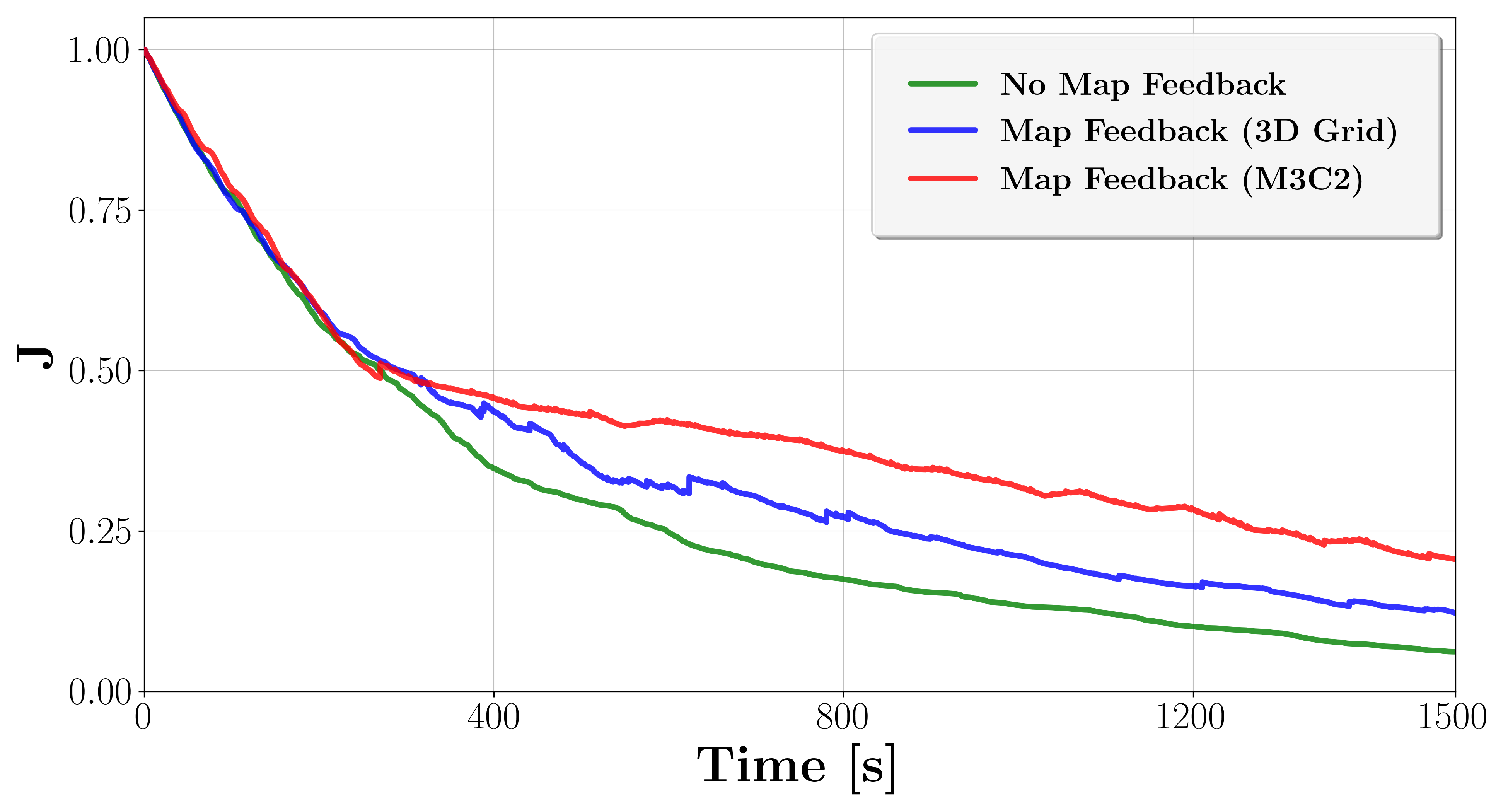}
    \caption{Evolution of the global objective $J$ comparing the no map feedback case, the 3D Grid map feedback method, and the M3C2 map feedback method.}
    \label{fig:j_evolution}
\end{figure}

\subsection{Simulation Setup}
\label{sec:sim_setup}

The simulation framework and implementation architecture are illustrated in Fig.~\ref{fig:system_architecture}. 
The system implementation follows a \textit{semi-distributed} architecture, where each drone executes the QP-based controller in \eqref{eq:qp_controller_final} locally, while a central computer handles global tasks such as generating a real-time 3D map of the scene, aggregating the evolving importance indices, and performing mesh change extraction. The mesh change extraction follows Sec.~\ref{sec:map_feedback}, using either the proposed 3D Grid method \eqref{eq:h2_grid} or the M3C2-based approach \eqref{eq:h2_m3c2} for comparison. 
All components are implemented on the ROS2 framework, which manages communication between the drones, the central computer, and the virtual environment.
The reconstruction scene is created in Unity, as shown in Fig.~\ref{fig:system_architecture}, featuring an indoor virtual 3D environment with a sofa, table, and wardrobe. This environment is connected to ROS2 through a Unity--ROS2 bridge, where rendered images are streamed to the NeuralRecon node for online reconstruction.

\begin{algorithm}
\caption{Coverage-Recon Algorithm}
\label{alg:overall}
\begin{algorithmic}[1]
\State Initialize $\delta_2 \gets 0$, generate initial 3D map
\While{system active}
    \State Compute $J$ using \eqref{eqn:j}
    \If{$J \geq J_{\text{th}}$} 
        \State Keep $\delta_2 = 0$ (map feedback inactive)
        \State Update $\phi_j$ using sensing performance: \eqref{eq:phi_j_update_h1}
    \Else
        \State Set $\delta_2 = \delta_2'$ (map feedback active)
        \State Update $\phi_j$ using \eqref{eq:phi_j_update_h1} and apply jump update: \eqref{eq:phi_j_update_h2}
    \EndIf
    \State Solve QP \eqref{eq:qp_controller_final} to compute $u_i$
    \State Update drone positions using $u_i$
    \State Continue image capture and reconstruction
\EndWhile
\end{algorithmic}
\end{algorithm}

The reconstruction target region $\mathcal{B}$ is defined as 
$\mathcal{B} = [-3.0,3.0]~\mathrm{m} \times [-3.0,3.0]~\mathrm{m} \times [0,2.0]~\mathrm{m}$, 
with the viewing angle space set to $\Theta_h \in [-\pi,\pi)$ and $\Theta_v \in [\pi/3,\pi/2]$. 
The virtual field $\mathcal{Q}$ is discretized into uniform cells of size $0.3~\mathrm{m} \times 0.3~\mathrm{m} \times 0.3~\mathrm{m} \times 0.3~\mathrm{rad} \times 0.3~\mathrm{rad}$, resulting in $m = 1.5\times10^7$ cells. 
The flight region $\mathcal{P}$ is bounded by $x_{\text{min}}=-4.0$~m, $x_{\text{max}}=4.0$~m, $y_{\text{min}}=-4.0$~m, $y_{\text{max}}=4.0$~m, $z_{\text{min}}=0.5$~m, and $z_{\text{max}}=5.0$~m. 
The parameter $d_{\min}$ in (\ref{eq:b_i,ca}) is set to 
$d_{\min}=1.0$~m.
In addition to the conditions 
(\ref{eq:cbf_gimbal_pitch}) and (\ref{eq:b_i,ca}), we add an additional inequality constraint to prevent drones from leaving the region $\mathcal{P}$.

The drone camera parameters are configured to match a DJI Mini 3, with a 35~mm equivalent focal length of 24~mm and an image resolution of 1920$\times$1080 pixels. 
The frame rate for image capture is fixed at 20~fps. 
Controller-related parameters for the sensing performance function are set to $D=1.0$, $\sigma_1=0.07$, $\sigma_2=0.095$, $\sigma_3=0.3$, and $\sigma_4=0.4$, while other control gains are $\delta_1=3.0$, $\delta_2'=1.0$, $\kappa=0.7$, $a_1 =1.0$, $a_2=1.0$, $a_3=1.0$ , $\epsilon=0.01$, and $\gamma=0.012$. 

To manage the computational load, multi-threaded processing is used for tasks such as mesh change detection and importance index updates. 
Simulations are executed on a workstation equipped with an Intel Core i9-14900KF CPU (24 cores, 32 threads, max 6~GHz), 64~GB RAM, and an Nvidia GeForce RTX 4090 GPU. 
Control updates and data exchanges are performed at 20~Hz to ensure real-time operation.

\subsection{Evaluation Metrics}
\label{sec:sim_metrics}

\begin{figure}[t]
    \centering
    \includegraphics[width=1.0\columnwidth]{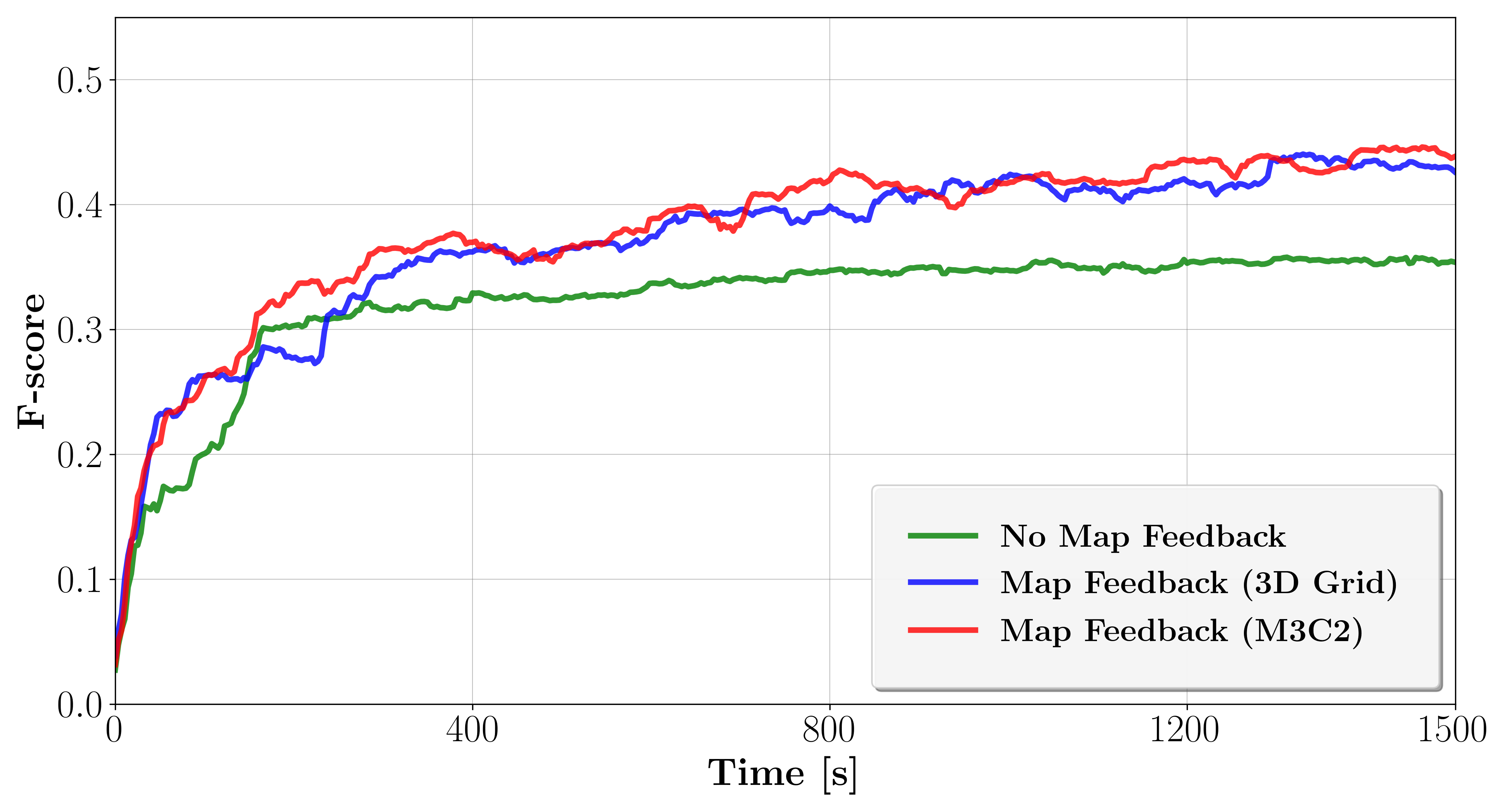}
    \caption{Evolution of the reconstruction F-score comparing the no map feedback case, the 3D Grid map feedback method, and the M3C2 map feedback method.}
    \label{fig:fscore_evolution}
\end{figure}

We evaluate the reconstructed mesh $\hat{\mathcal{G}}$ against a ground-truth mesh $\mathcal{G}$ using the F-Score metric, following the procedure in~\cite{murez2020atlas}. 
Let $d(\cdot,\cdot)$ denote the closest-point distance between two meshes. 
Precision and Recall are defined as 
$\mathrm{Precision} = |\{x \in \hat{\mathcal{G}} : d(x,\mathcal{G}) < \eta \}| / |\hat{\mathcal{G}}|$ 
and 
$\mathrm{Recall} = |\{y \in \mathcal{G} : d(y,\hat{\mathcal{G}}) < \eta \}| / |\mathcal{G}|$, 
where $\eta$ is the distance threshold for determining a correct correspondence.  
The F-Score is then given by
\begin{equation*}
    \mathrm{F\mbox{-}Score}(\eta) = 
    \frac{2 \cdot \mathrm{Precision} \cdot \mathrm{Recall}}
    {\mathrm{Precision} + \mathrm{Recall}}.
\end{equation*}
In our evaluation, we set $\eta = 0.05~\mathrm{m}$, as it is commonly used in 3D reconstruction benchmarks~\cite{murez2020atlas}.

The Unity mesh used as ground truth is intentionally sparse for rendering efficiency, whereas NeuralRecon typically produces a denser mesh. 
To avoid biasing Precision/Recall due to vertex-density mismatch, we uniformly resample the ground-truth mesh using \emph{Meshlab}'s \emph{Uniform Mesh Resampling} tool~\cite{cignoni2008meshlab}, which applies volumetric resampling followed by Marching Cubes reconstruction~\cite{lorensen1998marching}. 
This ensures that the ground-truth mesh has a comparable vertex density to NeuralRecon's output, enabling fair evaluation of accuracy.

\subsection{Single-Drone Evaluation and Baseline Comparison}
\label{sec:sim_single}

We first conduct single-drone evaluations to focus on the impact of the online map feedback mechanism in the proposed Coverage-Recon framework on the quality of the reconstructed 3D map. 
This subsection consists of two parts.  
First, we compare the proposed Coverage-Recon algorithm using the 3D Grid method~\eqref{eq:h2_grid} and the M3C2 method~\eqref{eq:h2_m3c2} against the case with no map feedback, which corresponds to setting $\delta_2 = 0$ throughout the mission. 
Next, we benchmark our method against representative baseline planners~\cite{torres2016coverage, dan2020control, shimizu2021angle, lu2024angle} to demonstrate the advantages of integrating online map feedback into the coverage control framework. 

\begin{figure*}[t]
    \centering
    \includegraphics[width=0.95\textwidth]{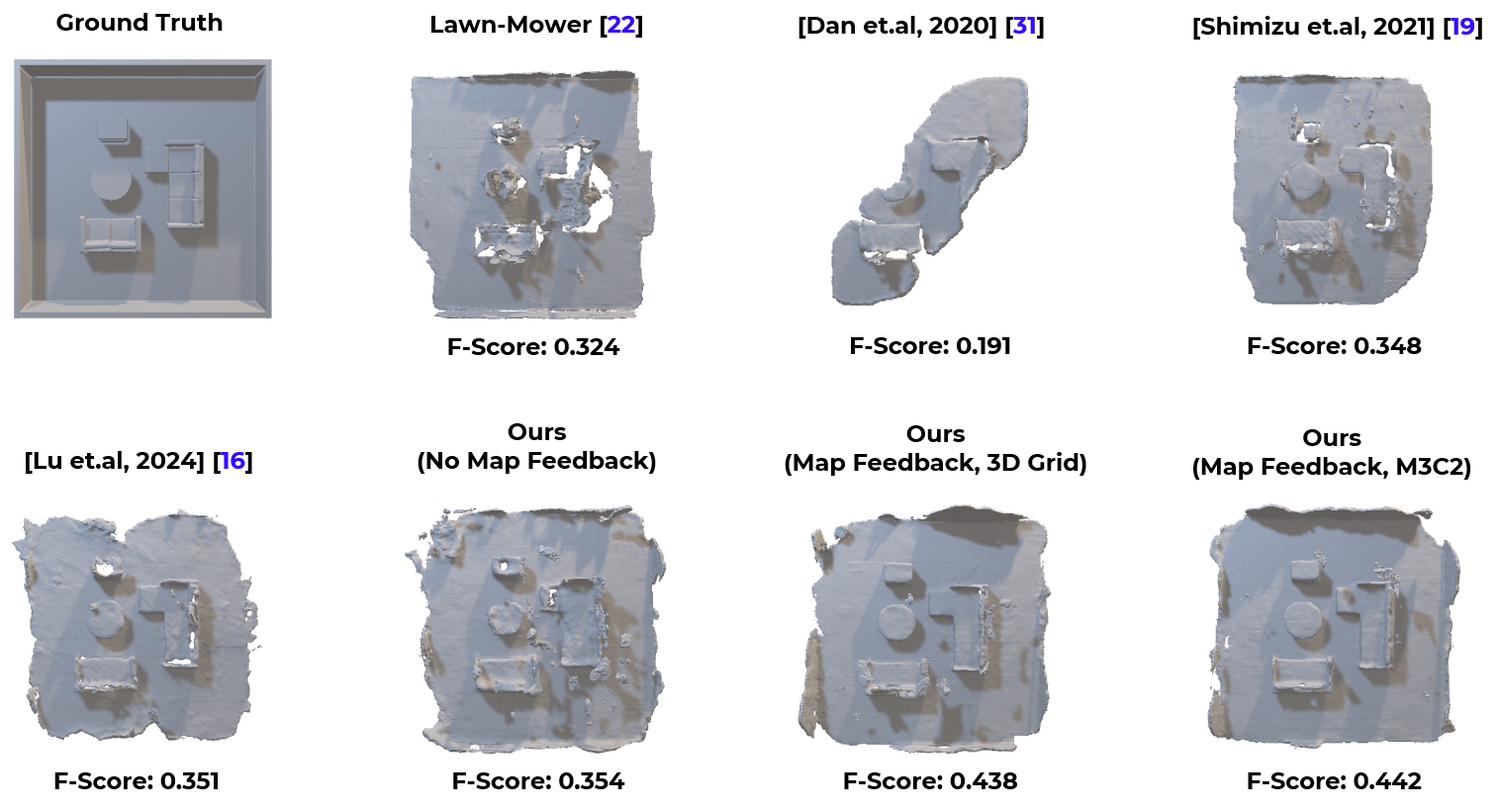}
    \caption{Final reconstruction meshes for different methods and parameter settings. The top-left image shows the ground truth, and the F-Score shown below each result indicates its reconstruction accuracy compared to the ground truth.}
    \label{fig:mesh_jth}
\end{figure*}

\vspace{6pt}
\noindent\textbf{1) Map Feedback vs. No Map Feedback.}  
We first evaluate the effect of online map feedback on reconstruction quality by comparing three cases: (i) a controller without map feedback ($\delta_2 = 0$), (ii) the proposed Coverage-Recon algorithm using the 3D Grid feedback mechanism~\eqref{eq:h2_grid} with $J_{\text{th}} = 50\%$, and (iii) the Coverage-Recon algorithm using the higher-resolution M3C2 feedback mechanism~\eqref{eq:h2_m3c2} with the same threshold.  
The threshold value $J_{\text{th}} = 50\%$ is empirically chosen through trial and error, while systematic selection of this parameter is left as future work.  
A single drone ($n=1$) is deployed from the same initial position $p_i = [0.0\ 1.0\ 2.0]^T \, \mathrm{m}$ in all cases.

Figure~\ref{fig:sim_snapshot} shows snapshots of the simulation using the map feedback controller with 3D Grid method.  
At $t = 5 \, \mathrm{s}$, the entire field is almost red with $\phi_j = 1$, indicating that no regions have been observed.  
As the mission progresses, $\phi_j$ decreases and shifts toward blue as more areas are captured, reaching near zero by $t = 1495 \, \mathrm{s}$.  
The bottom row illustrates the reconstruction process, where the mesh becomes progressively more complete and detailed over time.

The temporal evolution of the global objective $J(t)$ and the reconstruction accuracy F-Score$(t)$ for the three cases are presented in Figs.~\ref{fig:j_evolution} and~\ref{fig:fscore_evolution}, respectively.  
In Fig.~\ref{fig:j_evolution}, the controller without map feedback produces a monotonic decrease of $J$, while controllers with map feedback show discrete jumps whenever significant mesh changes are observed after $J \leq J_{\text{th}}$.  
These jumps gradually diminish as the mission approaches completion, with all settings converging near $J = 0$ at $t = 1500 \, \mathrm{s}$.  
From Fig.~\ref{fig:fscore_evolution}, we observe that both controller with map feedback consistently improves reconstruction accuracy, achieving higher F-Score values than the no-feedback case.  

The final reconstruction meshes for each setting are shown in Fig.~\ref{fig:mesh_jth}.  
Without map feedback, the reconstructed scene shows some missing pieces and less refined details compared to the versions with map feedback.  
As map feedback is introduced, the reconstruction quality improves, resulting in meshes that more closely resemble the ground truth, as also reflected in higher F-Score values.  

The M3C2 method achieves a slightly higher final F-Score of 0.442 compared to 0.438 for the 3D Grid method, owing to its finer geometric resolution. 
However, this improvement comes at the cost of a much higher computational load, with an average update time of 1.8585~s per cycle versus only 0.2131~s for the 3D Grid method.  
The lightweight 3D Grid approach provides nearly identical reconstruction quality while keeping feedback computations efficient, making it more practical for real-time multi-drone deployment.  
Accordingly, the map feedback mechanism employing the 3D Grid method is adopted for all subsequent simulations and experiments.

\vspace{6pt}
\noindent\textbf{2) Comparison with Baseline Planners.}  
We next compare the Coverage-Recon algorithm using the 3D Grid feedback method against several common baseline strategies.
First, we consider the lawnmower method, implemented similarly to~\cite{torres2016coverage}, where the drone follows a preset sweeping path with a fixed downward-facing camera at $90^\circ$ (See Figure~\ref{fig:lawn_mower}).
We also compare with other coverage-control-based approaches~\cite{dan2020control, shimizu2021angle, lu2024angle}, all operating without map feedback.
For the cases of~\cite{dan2020control} and~\cite{shimizu2021angle}, the drones also use fixed downward-facing cameras at $90^\circ$.

Figure~\ref{fig:mesh_jth} shows the final reconstructed meshes and the corresponding F-Score values.  
The lawnmower method~\cite{torres2016coverage} follows a preset sweeping path that provides basic coverage, but its fixed downward-facing camera limits viewpoint diversity, especially for complex structures, resulting in an incomplete reconstruction with an F-Score of 0.324.  
The persistent coverage method~\cite{dan2020control} yields the lowest F-Score of 0.191, as it continuously moves to maintain coverage but cannot guarantee that every point within the scene is visited.
This leads to repetitive motion along a narrow diagonal path, leaving many regions of the scene unobserved.  
The angle-aware coverage control with a fixed camera view~\cite{shimizu2021angle} addresses this issue by ensuring that every point in the scene is eventually visited, producing a higher F-Score of 0.348.  
Nevertheless, the fixed downward-facing camera still limits viewpoint diversity, making it difficult to accurately capture highly sloped surfaces such as furniture sides. The method of~\cite{lu2024angle} improves it by enabling active camera rotation, achieving an F-Score of 0.351, close to our approach without map feedback (0.354).
In general, our proposed method with map feedback achieves the highest reconstruction accuracy, with an F-Score of 0.438.
Unlike other approaches that only consider whether a point has been visited, the feedback from the ongoing reconstruction provides valuable information to identify regions with insufficient reconstruction, allowing the drone to focus on and improve those areas.

\begin{figure}[t]
    \centering
    \includegraphics[width=0.88\columnwidth]{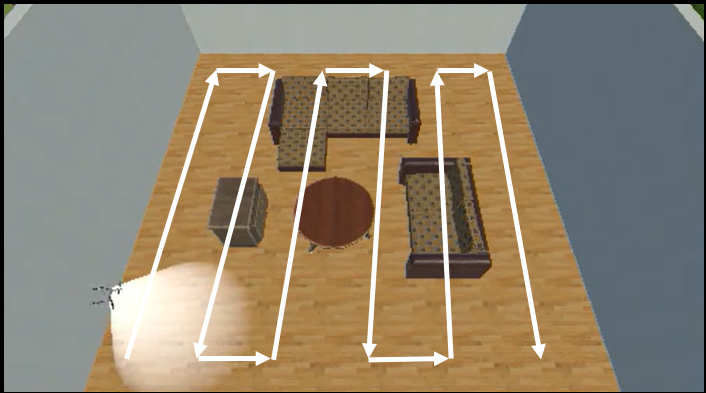}
    \caption{Illustration of the lawnmower path, similarly based on~\cite{torres2016coverage}, where the drone follows a preset sweeping trajectory with a fixed downward-facing camera at $90^\circ$. A demonstration video of this baseline strategy is available at \href{https://www.youtube.com/watch?v=hZplItVmM4M}{this link}.}
    \label{fig:lawn_mower}
\end{figure}

\subsection{Multi-Drone Evaluation}
\label{sec:sim_multi}

To demonstrate the scalability of the proposed framework, we employ the Coverage-Recon controller with the adopted 3D Grid feedback configuration ($J_{\text{th}} = 50\%$) and vary the number of drones $n \in \{1, 2, 4\}$.  
All other parameters, including drone speed, camera field of view, and safety constraints, are kept identical across team sizes to ensure a fair comparison.

As shown in Fig.~\ref{fig:multi_j}, using multiple drones significantly accelerates the decrease of the global objective $J(t)$, leading to faster mission completion.  
For example, $J \approx 0.25$ is reached at approximately $t = 850 \, \mathrm{s}$ with one drone, $t = 620 \, \mathrm{s}$ with two drones, and $t = 320 \, \mathrm{s}$ with four drones.  
The reconstruction accuracy, measured by the final F-Score, remains comparable across team sizes, with values of 0.438 for one drone, 0.422 for two drones, and 0.423 for four drones, indicating that the framework scales efficiently while preserving reconstruction quality.
These results indicate that Coverage-Recon can be extended to multi-drone systems, supporting cooperative exploration and more efficient image sampling in larger-scale reconstruction tasks.  
A video of the multi-drone simulation can be viewed \href{https://www.youtube.com/watch?v=M1L07WQFRPI}{here}.

\begin{figure}[t]
    \centering
    \includegraphics[width=0.98\columnwidth]{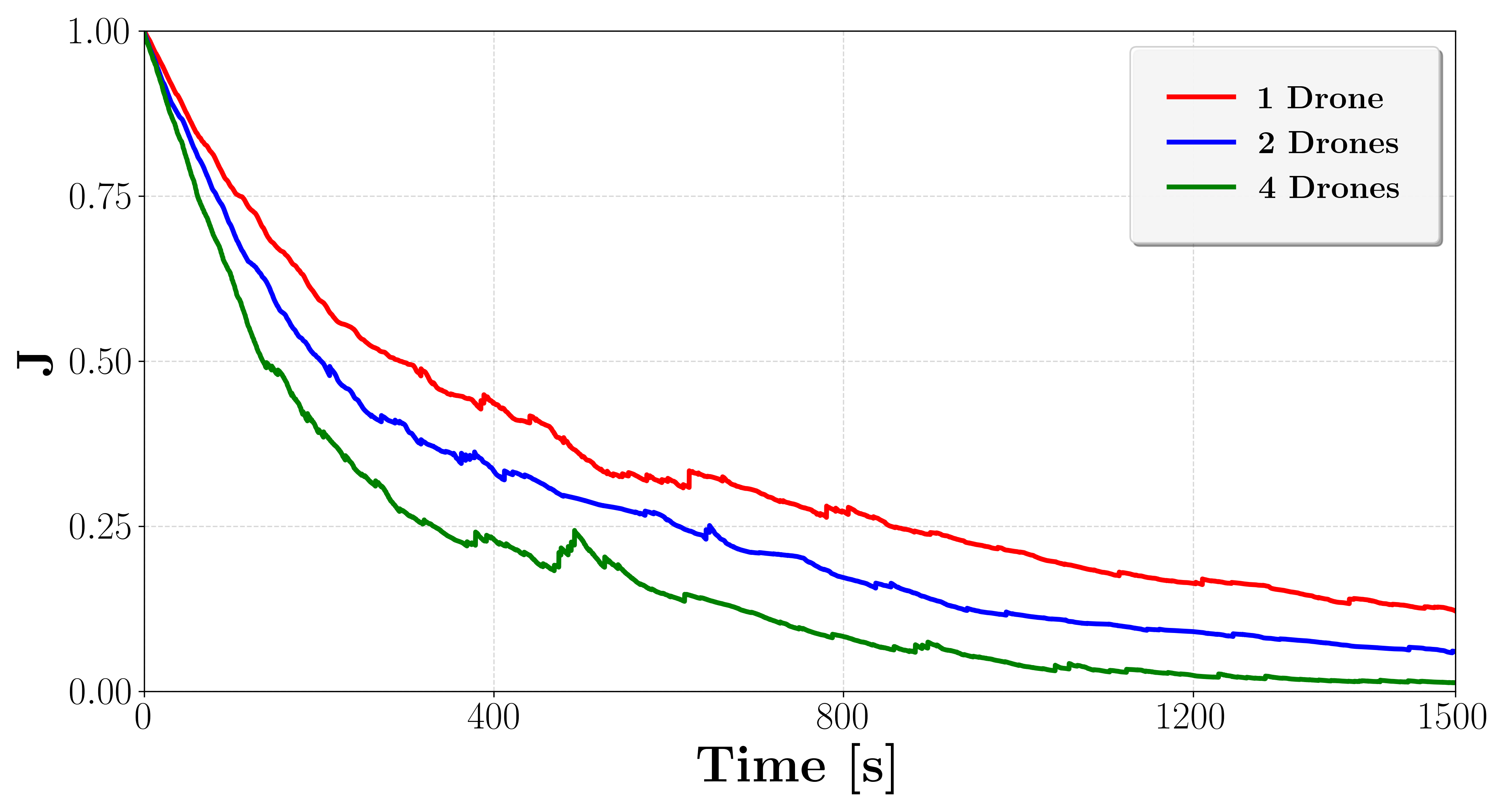}
    \caption{Evolution of the global objective $J$ for different drone team sizes ($n = 1, 2, 4$) using the Coverage-Recon controller with 3D Grid feedback and $J_{\text{th}} = 50\%$.}
    \label{fig:multi_j}
\end{figure}


\section{Experimental Results}
\label{sec:exp}

To validate the proposed Coverage-Recon framework, indoor flight experiments were conducted at the Science Tokyo Robot Zoo Sky testbed~\cite{suenaga2022experimental}.  
This section describes the experimental configuration and presents the results demonstrating the real-time feasibility and map quality advantage of the proposed framework.


\begin{figure*}[t]
    \centering
    \begin{minipage}[t]{0.48\textwidth}
        \centering
        \refstepcounter{figure}\addtocounter{figure}{-1} 
        \includegraphics[width=\linewidth]{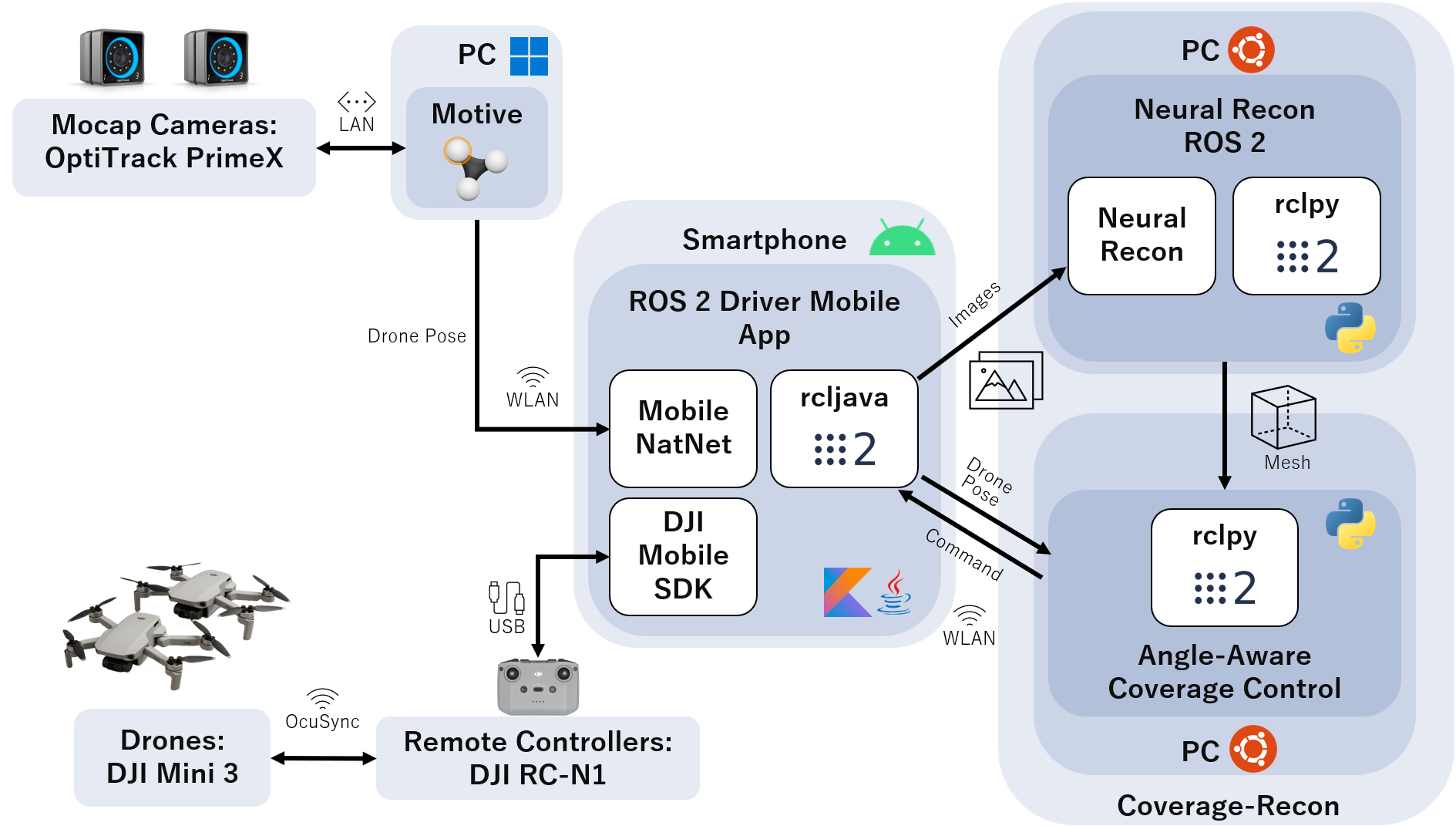}
        \label{fig:exp_setup_combined_a}
    \end{minipage}\hfill
    \begin{minipage}[t]{0.48\textwidth}
        \centering
        \refstepcounter{figure}\addtocounter{figure}{-1} 
        \includegraphics[width=\linewidth]{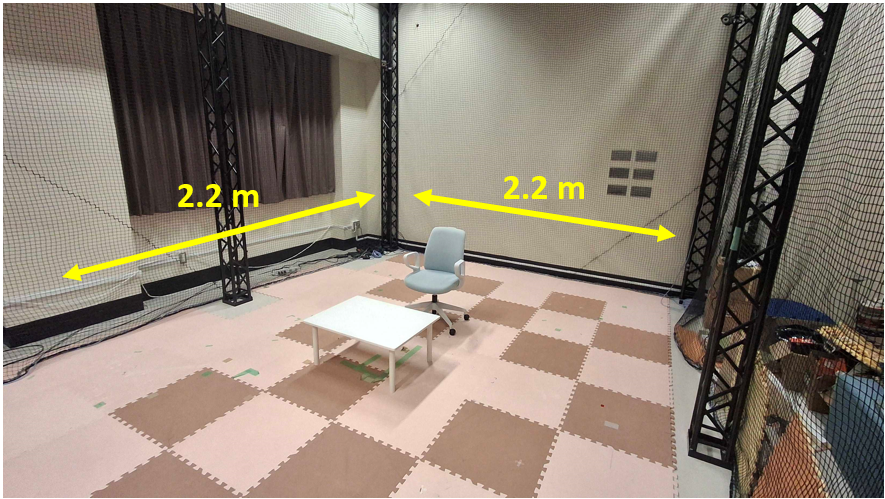}
        \label{fig:exp_setup_combined_b}
    \end{minipage}

    \caption{(a) Experimental configuration setup. (b) Science Tokyo Robot Zoo Sky testbed used for indoor flight experiments.}
    \label{fig:exp_setup_combined}
\end{figure*}

\subsection{Experimental Setup}
\label{sec:exp_setup}

The experimental configuration is shown in Fig.~\ref{fig:exp_setup_combined_a}(a). A DJI Mini 3 UAV with a forward-facing RGB camera is used for image capture and flight control. The UAV connects to a DJI RC-N1 controller via OcuSync, which links to an Android smartphone over USB. The smartphone runs a custom ROS 2 driver built with the DJI Mobile SDK, bridging control commands, pose feedback, and images to the ROS 2 network. Motion capture (MoCap) data from an OptiTrack PrimeX system, offering sub-millimeter precision, is streamed via Motive and the NatNet protocol for accurate tracking. The entire framework runs on a PC equipped with an Intel Core i9-13900K CPU, 64\,GB RAM, and an NVIDIA RTX 4090 GPU.

The experimental field can be seen in Fig.~\ref{fig:exp_setup_combined_b}(b). Here, we also implement a flight-boundary region CBF to prevent the UAV from approaching the walls of the testing area. All other parameters are set the same as in the simulation, except for the reconstruction region $\mathcal{B}$, the flight boundary $\mathcal{P}$, the maximum input $u_i^{\max}$, and related physical limits. In this experiment, we compare two configurations: the baseline case without map feedback and the proposed Coverage-Recon method with map feedback using the 3D Grid approach and $J_{\text{th}}=50\%$.

\subsection{Results and Analysis}
\label{sec:exp_results}

\begin{figure*}[t]
    \centering
    \includegraphics[width=\textwidth]{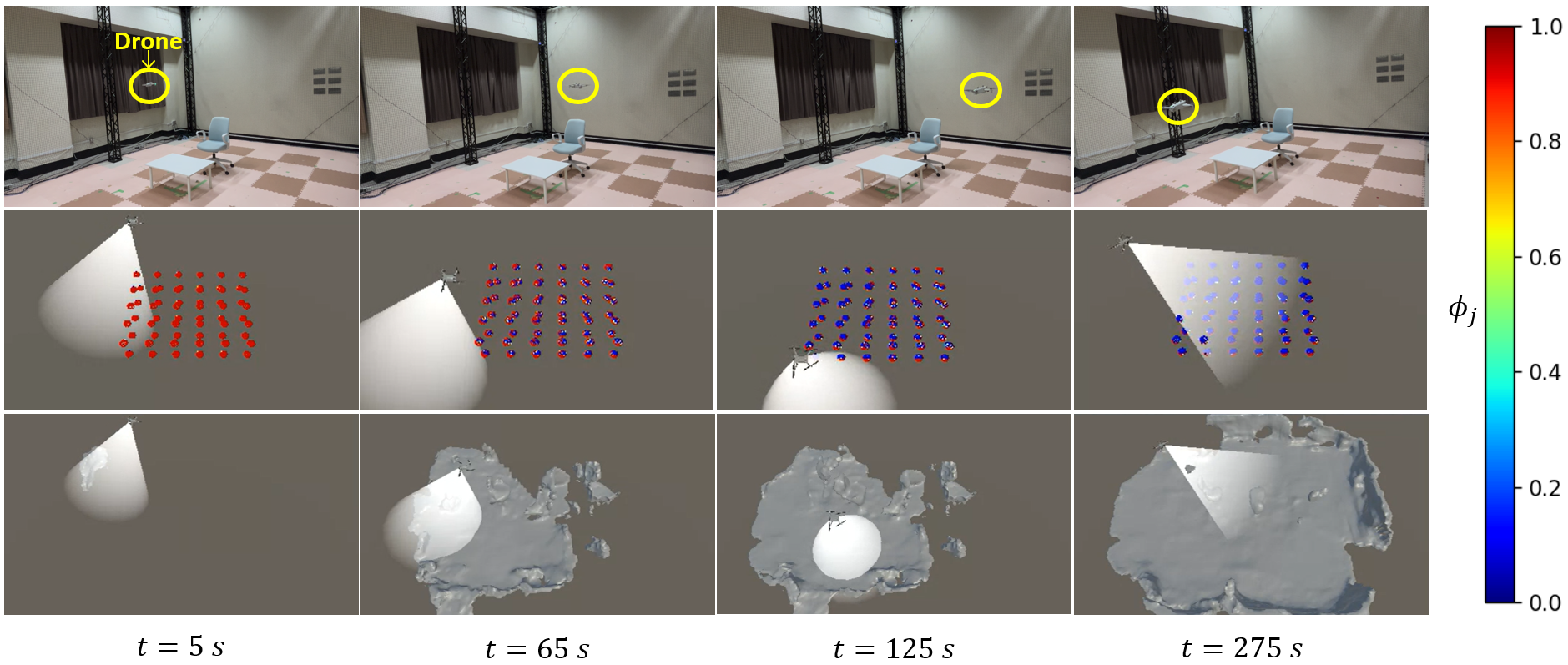}
    \caption{Experiment snapshots of the proposed Coverage-Recon framework with a single drone, using the 3D Grid method for mesh-change quantification and a threshold $J_{\text{th}} = 50\%$. 
    The top row shows the actual scene, the middle row the importance index $\phi_j$ (red: high, blue: low), and the bottom row the reconstructed map. 
    A video of the experiment can be viewed \href{https://www.youtube.com/watch?v=vX7Z6vx1rc8}{here}.}
    \label{fig:exp_snapshot}
\end{figure*}

The indoor flight experiment was conducted using a single drone to compare the proposed Coverage-Recon method with map feedback (3D Grid, $J_{\text{th}} = 50\%$) against the no-feedback case.  
Snapshots of the experiment are shown in Fig.~\ref{fig:exp_snapshot}.  
As the drone flies and observes the environment, the importance indices $\phi_j$ gradually shift from red to blue, indicating that previously unobserved regions are being covered.  
Meanwhile, the reconstructed mesh is iteratively updated and becomes progressively more complete as the mission proceeds.  
The final reconstruction results for both methods are presented in Fig.~\ref{fig:exp_mesh}.  
Qualitatively, the proposed method with map feedback produces a more refined and complete reconstruction, particularly around the table and chair, closely matching the actual scene.  
In contrast, the no-feedback case results in a coarser and less complete model with missing surfaces.  
These results confirm that incorporating online map feedback effectively guides the UAV toward under-observed areas, improving overall scene coverage and reconstruction quality in real time.

\begin{figure*}[t]
    \centering
    \includegraphics[width=0.9\textwidth]{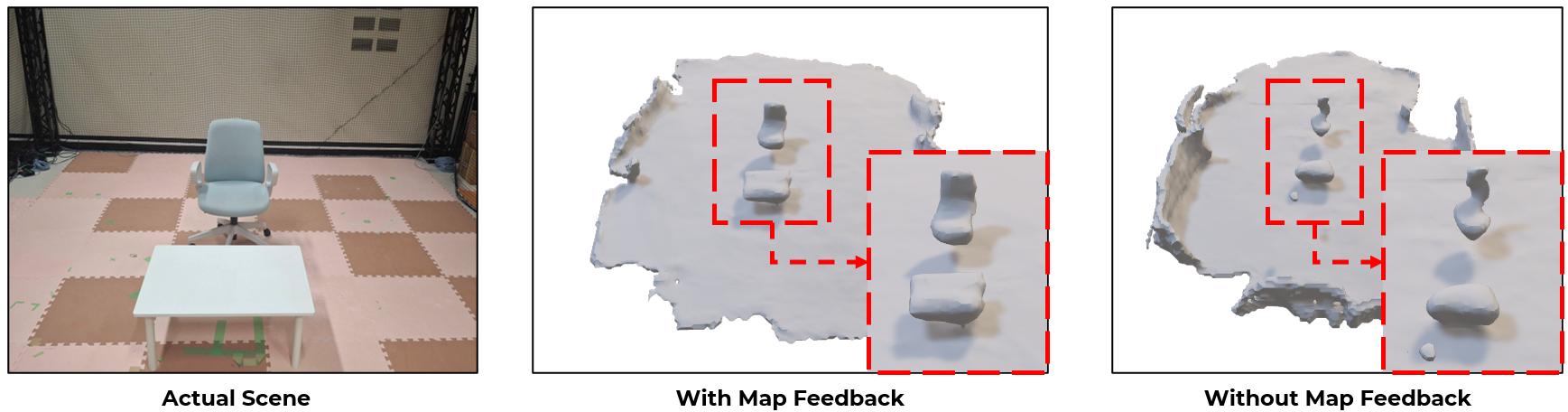}
    \caption{Final reconstruction meshes from the experimental verification. The proposed method with map feedback yields a more complete and refined reconstruction compared to the no-feedback case.}
    \label{fig:exp_mesh}
\end{figure*}


\section{Conclusion and Future Works}

The proposed \textit{Coverage-Recon} framework integrates online map feedback into coverage control to improve 3D reconstruction quality.  
Simulation results confirmed that the proposed algorithm with map feedback achieved better reconstruction quality than baseline approaches, as evidenced by higher F-Score values.  
Experiments further demonstrated the real-time feasibility and qualitative advantage of the proposed framework compared to the no-feedback case.  
Future work will explore adaptive feedback tuning and extensions toward radiance-field-based reconstruction methods such as Gaussian splatting.

\section*{Acknowledgments}
The authors would like to express their sincere gratitude to Mr. Daisuke Ichihashi and Dr. Kuniaki Uto for their valuable advice and support throughout this research.

\bibliographystyle{IEEEtran}
\bibliography{main}

\section{Biographies}

\begin{IEEEbiography}[{\includegraphics[width=1in,height=1.25in,clip,keepaspectratio]{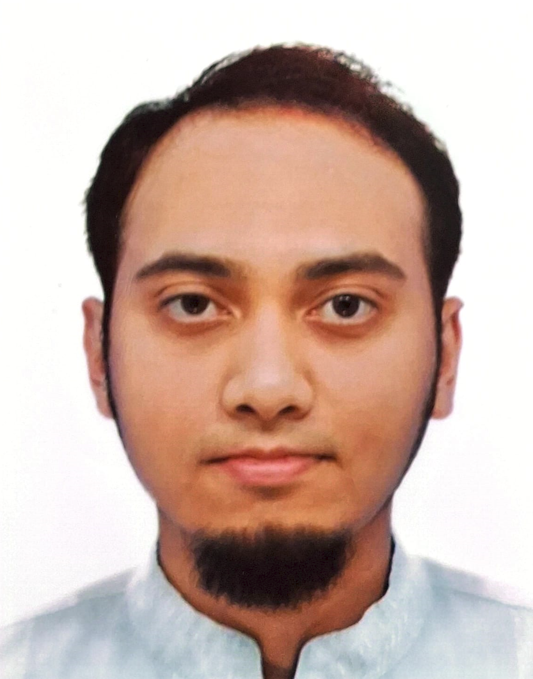}}]{Muhammad Hanif}
(Graduate Student Member, IEEE) received the B.S. degree in Electrical Engineering from Institut Teknologi Bandung, Indonesia, in 2018, and the M.Eng. and Ph.D. degrees in Systems and Control Engineering from the Institute of Science Tokyo (formerly Tokyo Institute of Technology), Japan, in 2022 and 2025, respectively. His research interests include multi-robot control and aerial robotics.
\end{IEEEbiography}

\begin{IEEEbiography}[{\includegraphics[width=1in,height=1.25in,clip,keepaspectratio]{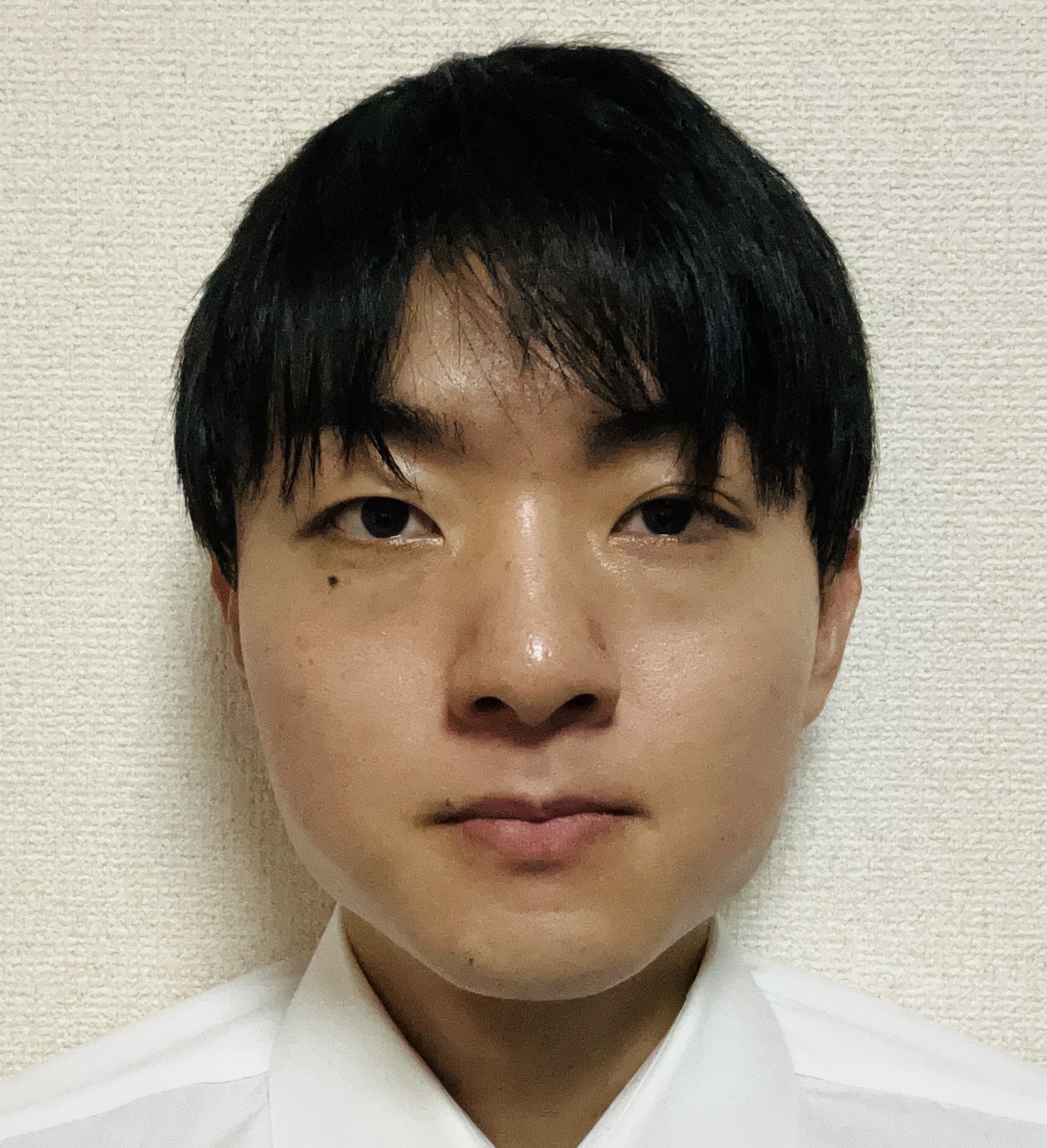}}]{Reiji Terunuma}
(Graduate Student Member, IEEE) received the B.Eng. and M.Eng. degrees in Systems and Control Engineering from the Institute of Science Tokyo (formerly Tokyo Institute of Technology), Japan, in 2023 and 2025, respectively, where he is currently pursuing the Ph.D. degree. His research interests include analysis and control of robotic systems.
\end{IEEEbiography}

\begin{IEEEbiography}[{\includegraphics[width=1in,height=1.25in,clip,keepaspectratio]{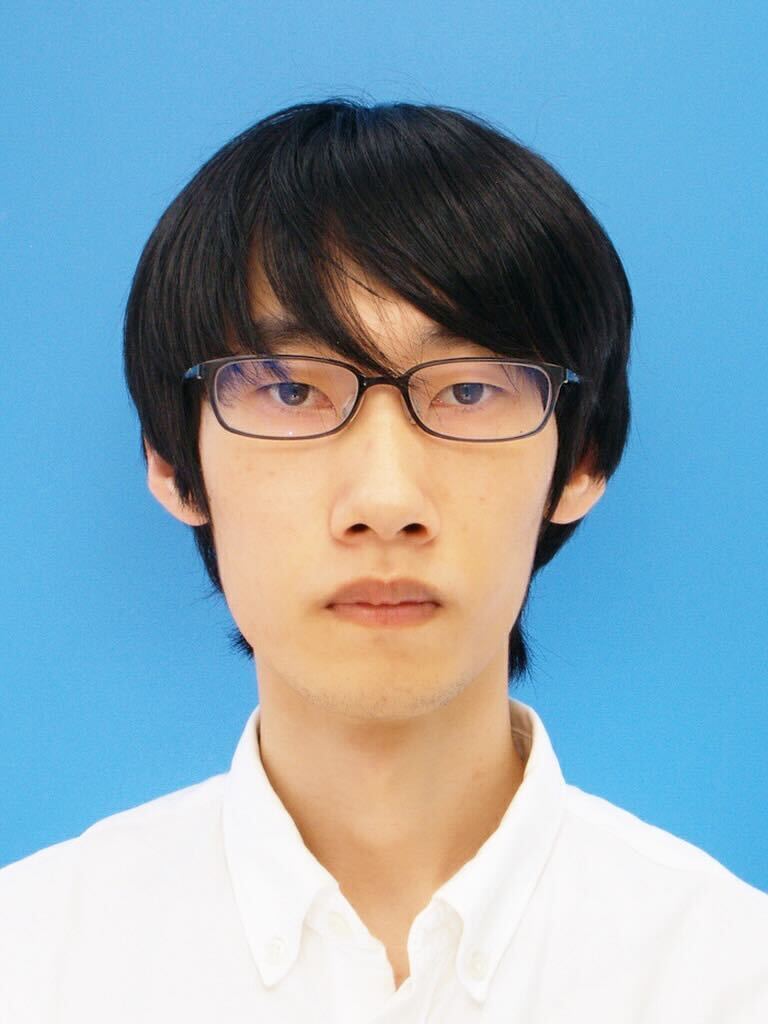}}]{Takumi Sumino}
received the B.Eng. and M.Eng. degree from the department of Systems and Control Engineering, Institute of Science Tokyo (formerly Tokyo Institute of Technology), Japan, in 2022 and 2024, respectively. His research interests include human-robot collaborations and networked robotics.
\end{IEEEbiography}


\begin{IEEEbiography}[{\includegraphics[width=1in,height=1.25in,clip,keepaspectratio]{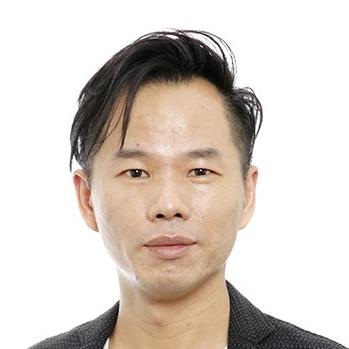}}]{Kelvin Cheng} received his Ph.D. degree in computer science from The University of Sydney, Australia in 2009. Since then he has worked at Australia's national scientific agency CSIRO and at the National University of Singapore, focusing on the area of Human-Computer Interaction. In 2017, he joined Rakuten Institute of Technology, part of Rakuten Group, Inc., and subsequently Rakuten Mobile Inc. from 2020. From 2018 to 2019, he was a visiting researcher with Waseda University, Japan, and from 2021 to 2024, he was a visiting researcher with The University of Tsukuba, Japan. He is currently an R\&D Manager and is actively creating future immersive consumer experiences and business use cases, together with a variety of cutting-edge technologies, and mobile technologies such as XR, 5G, Edge Computing, generative AI, drones, and autonomous robots/vechicles.
\end{IEEEbiography}

\begin{IEEEbiography}[{\includegraphics[width=1in,height=1.25in,clip,keepaspectratio]{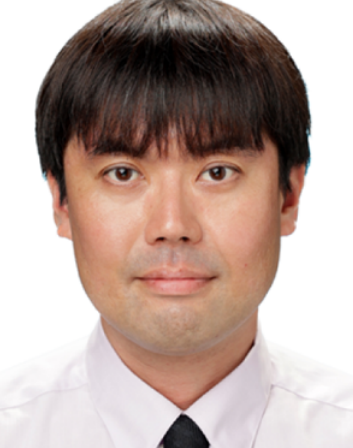}}]{Takeshi Hatanaka}
(Senior Member, IEEE) received the Ph.D. degree in applied mathematics and physics from Kyoto University, Kyoto, Japan, in 2007. He was a Faculty Member with Tokyo Institute of Technology, Tokyo, Japan, and Osaka University, Suita, Japan. Since October 2024, he has been a Professor with the Institute of Science Tokyo. He is the co-author of the \textit{Passivity-Based Control and Estimation in Networked Robotics} (Springer, 2015). His research interests include cyber-physical-human systems.

Dr. Hatanaka was a recipient of the Kimura Award in 2017, Pioneer Award in 2014, Pioneer Technology Award in 2024, the Outstanding Book Award in 2016, the Control Division Conference Award in 2018, the Takeda Prize in 2020, and the Outstanding Paper Awards in 2009, 2015, 2020, 2021, 2023, and 2025 all from The Society of Instrumental and Control Engineers, and the IFAC 2023 Application Paper Prize Finalist. He is serving/served as the Deputy EiC for the \textit{Annual Reviews in Control}, a Senior Editor and an Associate Editor for \textit{IEEE Transactions on Control Systems Technology}, the \textit{Mechatronics}, the \textit{Advanced Robotics}, and an IEEE CSS Conference Editorial Board Member.
He is a Co-general Chair of IFAC CPHS 2026.
\end{IEEEbiography}

 




\vfill

\end{document}